\def\paperTitle{Muon Is Theoretically Wrong for Convolutions, but Empirically Effective}

\def\authorBlock{
\hspace{0cm}
\textbf{Thibaut Boissin}$^{1,2}$
\hspace{1mm}
\textbf{Thomas Massena}$^{2,3}$
\hspace{1mm}
\textbf{Mathieu Serrurier}$^{2}$
\vspace{1mm}
\textbf{Franck Mamalet}$^{1}$
\hspace{3mm}
\\
$^1$ Institut de Recherche Technologique Saint-Exupery, France\\
$^2$ IRIT, France\\
$^3$ Innovation \& Research Division, SNCF, France\\
}

\newif\ifreview 

\newif\ifarxiv \newcommand{\arxiv}{\arxivtrue}
\newcommand{\preprint}{\arxiv}
\newif\ifcamera 

\newif\ifrebuttal

\preprint

\documentclass{article} 

\usepackage{iclr2027_conference,times}
\ifreview
\else
  \ifarxiv
    \iclrfinalcopy
  \else
    \ifcamera
        \iclrfinalcopy
    \else
      \ifrebuttal
        not defined
      \else
        not defined
      \fi
    \fi
  \fi
\fi

\usepackage{graphicx}	
\usepackage{amsmath}	
\usepackage{amssymb}	
\usepackage{booktabs}
\usepackage{times}
\usepackage{microtype}
\usepackage{epsfig}
\usepackage{caption}
\usepackage{float}
\usepackage{placeins}
\usepackage{color, colortbl}
\usepackage{stfloats}
\usepackage{enumitem}
\usepackage{tabularx}
\usepackage{xstring}
\usepackage{multirow}
\usepackage{xspace}
\usepackage{url}
\usepackage{subcaption}
\usepackage{xcolor}
\usepackage[hang,flushmargin]{footmisc}
\usepackage{makecell}

\usepackage{amsthm}
\usepackage{thmtools,thm-restate}
\usepackage{amsmath}
\usepackage{hyperref}
\usepackage{cleveref}
\usepackage{amsfonts}
\usepackage{amssymb}
\usepackage{dsfont}
\usepackage{mathtools}
\usepackage{algorithm}
\usepackage{algpseudocode}

\theoremstyle{plain}
\newtheorem{theorem}{Theorem}[section]
\newtheorem{proposition}[theorem]{Proposition}

\newtheorem{corollary}[theorem]{Corollary}
\theoremstyle{definition}

\theoremstyle{remark}
\newtheorem{remark}[theorem]{Remark}

% ...

\newcommand{\stief}{\mathrm{St}}
\newcommand{\convPF}{\operatorname{ConvPF}}

\ifarxiv  \fi

\newcommand{\R}[1]{{%
    \textbf{%
        \ifstrequal{#1}{1}{\textcolor{red}{R#1}}{%
        \ifstrequal{#1}{2}{\textcolor{blue}{R#1}}{%
        \ifstrequal{#1}{3}{\textcolor{magenta}{R#1}}{%
        \ifstrequal{#1}{4}{\textcolor{teal}{R#1}}{%
                           \textcolor{cyan}{R#1}%
        }}}}%
    }%
}}

\newif\ifshowcomments
\showcommentstrue 

\ifshowcomments

\newcommand{\thomas}[1]{{\color{blue}Thomas:#1}}
\newcommand{\thib}[1]{{\color{magenta}Thib:#1}}
\newcommand{\franck}[1]{{\color{red}Franck:#1}}
\else
\newcommand{\thib}[1]{}
\newcommand{\thomas}[1]{}
\newcommand{\franck}[1]{}
\fi

\newif\ifshowauthors
\showauthorsfalse

\ifarxiv
  \showauthorstrue
\fi

\ifcamera
  \showauthorstrue
\fi

\title{\paperTitle}
\author{\authorBlock}

\begin{document}

\maketitle

\ifarxiv
\lhead{}
\fi
\begin{abstract}

Muon, an optimizer known for its efficiency, has a clear interpretation for matrix-valued updates, but convolutional kernels are stored as four-dimensional tensors. Standard implementations reshape these tensors into matrices, a shortcut which breaks the theoretical understanding behind Muon. To investigate this,  
we formalize the corresponding optimization objective directly in convolutional operator geometry and
introduce Convolutional Newton-Schulz (Conv-NS), which approximates the polar factor in this geometry while preserving kernel support. When applied in fast training experiments, Conv-NS and reshape-based Muon are both computationally efficient and achieve comparable accuracy on CIFAR-10 and ImageNet classification tasks. However, as one could expect a theoretically aligned Conv-NS to outperform reshape-based Muon, we investigate this mismatch between practice and theoretical understanding, with the hypothesis that exact convolutional orthogonalization may overconstrain updates. 
These findings highlight Muon's strong practical performance while opening directions for its further development on convolutions.
\ifshowauthors
Our code is publicly available at 
\href{https://github.com/thib-s/muonconv-cifar10-airbench}{github conv-muon}.
\else 
\fi

\end{abstract}

\section{Introduction}
\label{sec:intro}

A recent line of work in deep learning optimization departs from element-wise update rules such as SGD and AdamW and instead exploits the matrix structure of neural network parameters. This perspective has led to influential methods ranging from preconditioned second-order optimizers such as Shampoo~\citep{gupta2018shampoo,anil2020scalable}, which accumulate Kronecker-factored curvature estimates, to first-order methods such as Muon~\citep{jordan_muon_2024,bernstein_deriving_2025}. Muon replaces the gradient of a linear layer by an approximation of its semi-orthogonal polar factor, effectively promoting well-conditioned, near-orthogonal updates at the operator level.
This optimizer has demonstrated competitive performance across tasks such as transformer pretraining~\citep{modded_nanogpt_2024,nanochat}, image diffusion model training~\citep{schaipp2025optimization}, and large language model training~\citep{ai_practical_2025,modded_nanogpt_2024,liu_muon_2025}. We provide a more detailed discussion of related work in \Cref{sec:related}.

While this construction is natural for linear layers, convolutional layers do not fit trivially into the same framework. In this paper, we study the use of Muon on 2D convolutions. A convolutional layer is parameterized by a compact kernel
$K \in \mathbb{R}^{c_{\mathrm{out}} \times c_{\mathrm{in}} \times k_h \times k_w}$,
but acts on activations through a much larger structured linear operator $\mathcal{T}(K)$, typically Toeplitz or doubly circulant depending on the boundary conditions, that need not be materialized during the forward pass to avoid redundant computations. In practice, current implementations apply Muon by reshaping the 4D kernel into a matrix of size
$\mathbb{R}^{c_{\mathrm{out}} \times (c_{\mathrm{in}}k_hk_w)}$ and orthogonalizing this reshaped matrix.

However, orthogonalizing the reshaped kernel does not generally correspond to orthogonalizing the equivalent linear transformation matrix (see~\cref{fig:problem-framing}, upper part).
The resulting update is therefore misaligned with the operator geometry of the layer being optimized.
We formalize this mismatch in \cref{sec:background-problem} and introduce the first convolution-aware construction that directly targets the polar factor of the induced operator while preserving the finite support of the kernel.
In this paper, we use orthogonalization to refer to procedures that drive the singular values of the convolutional operator toward one (as seen in~\cref{fig:problem-framing}, right). 

Independently from the less generally understood question of generalization~\citep{wilson_marginal_2017,balles_dissecting_2018,luo2018adaptive}, we run benchmarks to prioritize speed of convergence (\cref{subsec:cifar10-imagenet}). Surprisingly, reshape-based Muon, despite its structural mismatch, and our convolution-aware alternative Conv-NS achieve comparable performance on these benchmarks,
while both substantially outperforming SGD with momentum in compute-constrained settings. These results raise the following question: \textit{why does more faithful convolutional orthogonalization not improve optimization performance?}

The insights we provide are the following:

\begin{itemize}
    \item \textbf{Standard Muon overlooks convolutional geometry.} Applying Muon to convolutions typically involves reshaping each kernel into a matrix before orthogonalization. This does not generally orthogonalize the convolution operator.
    \item \textbf{Convolutional orthogonalization is computationally practical.} We introduce Conv-NS, which preserves kernel support and empirically approaches convolutional orthogonality with near-minimal gradient modification.
    \item \textbf{Better orthogonalization does not consistently improve training.} Conv-NS and reshape-based Muon achieve competitive performance on CIFAR-10 and ImageNet, while increasing Conv-NS iterations can degrade training performance.
    \item \textbf{Extending Muon to convolutions requires more than polar factorization.} We revisit the optimization principle underlying Muon: choosing an update that maximizes predicted loss reduction while limiting its effect on layer outputs. Reshaping does not enforce this limit, while convolutional orthogonality imposes a stronger constraint.
    \item \textbf{These limitations open directions for convolutional Muon.} Both implementations already achieve strong accuracy--runtime trade-offs in the tested settings, motivating updates that better balance predicted loss reduction and control of layer outputs.
\end{itemize}

\section{Muon Foundations and the Convolution Mismatch}
\label{sec:background-problem}

\paragraph{Notations}
For any matrices $X\in\mathbb{R}^{m\times n}$, we denote by $\|X\|_F$ its Frobenius norm, by $\langle X, Y \rangle_F$ the associated inner product, and by $\|X\|_{\mathrm{op}}$ its operator (or spectral) norm. The Stiefel manifold is denoted when $m\leq n$, by $
\stief(m,n):=\left\{X\in\mathbb{R}^{m\times n}: X X^\top=I\right\}$\footnote{the condition for column orthogonal matrices $n\leq m$ is $X^\top X =I$, with $I$ of size $n \times n$},
where $I$ denotes the $m\times m$ identity matrix. When the dimensions are clear from the context, we simply write $\stief$. We denote the squared Frobenius distance between two matrices $X,Y\in\mathbb{R}^{m\times n}$, by $\mathcal{L}_{\mathrm{frob}}(X,Y) = \| X - Y \|_F^2$.
\subsection{About the Muon Update Rule}
\label{subsec:muon-background}

\begin{figure*}[t]
    \centering
    \includegraphics[width=0.96\textwidth]{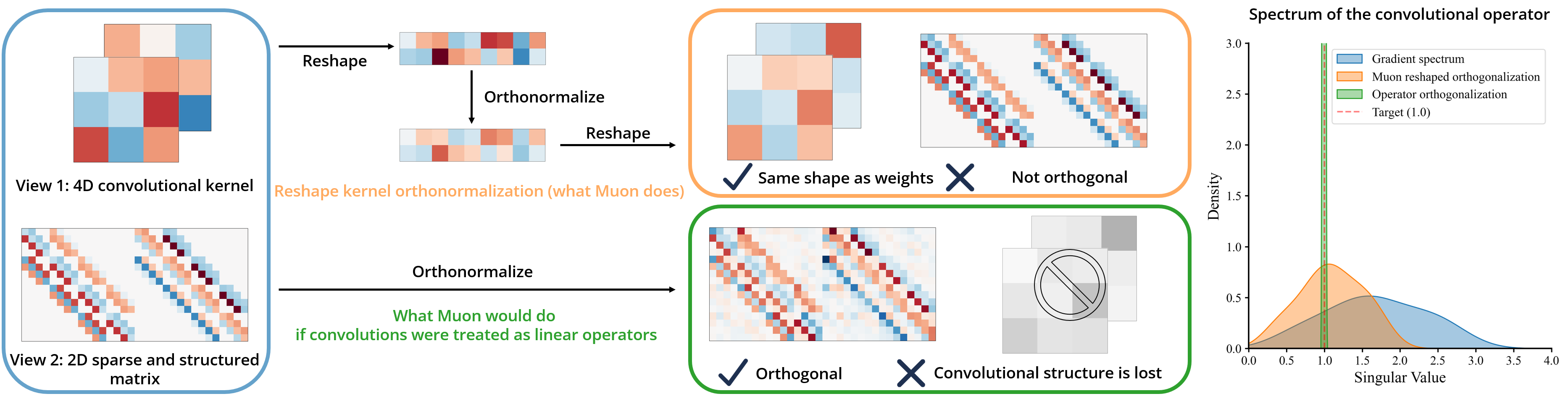}
    \caption{\textbf{Problem framing for Muon on convolutions.} Muon's current implementation (option 1) reshape the update and apply $\operatorname{PF}$, however this result in a non orthogonal convolution as observed on the singular value distibution. However there is no trivial solution since the direct application of $\operatorname{PF}$ to the equivalent linear operator leads to a an orthogonal matrix, but the structure of the convolution is lost in the process. The $3\times 3$ convolution with $c_i=2, c_o=1$ was sampled from a gaussian distribution.}
    \label{fig:problem-framing}
\end{figure*}

A dense hidden layer acts on activations via a weight matrix $W \in \mathbb{R}^{d_{\mathrm{out}} \times d_{\mathrm{in}}}$. Given a momentum or gradient matrix $G$, the Muon optimizer \citep{jordan_muon_2024} computes the update direction by approximately orthogonalize the gradient matrix.
If $G = U \Sigma V^\top$ is the singular value decomposition, then the Muon update is:
\begin{equation}
    \label{eq:PFoptimum}
    \Delta_W^\star = -\eta \operatorname{PF}(G) = -\eta U V^\top
\end{equation}
From an optimization perspective, the polar factor $\operatorname{PF}(G)$ can equivalently be viewed as the projection of $G$ onto the Stiefel manifold $\stief$,
and thus defined by a constrained minimization over the objective $\mathcal{L}_{\mathrm{frob}}(G,Q)$:
\begin{equation}
    \operatorname{PF}(G) = \underset{Q\in \stief}{\operatorname{argmin}} \quad \mathcal{L}_{\mathrm{frob}}(G,Q) \quad
    \label{eq:dense-polar-projection}
\end{equation}

Since computing an exact SVD at every step is computationally prohibitive, Muon practically approximates $\operatorname{PF}(G)$ using a small number of Newton-Schulz iterations~\citep{bjorck1971iterative, amsel_polar_2025, grishina_tight_2024}\footnote{Variants with fifth order matrix multiplication is often used}. For an appropriately scaled matrix $X$, the iteration:
\begin{equation}
    X \leftarrow (1+\beta)X - \beta X X^\top X
    \label{eq:dense-newton-schulz}
\end{equation}
drives the orthogonality penalty $\| X X^\top - I \|_F^2$ iteratively to zero, yielding the polar factor.

\subsection{How Standard Muon Handles Convolutions}
\label{subsec:how-muon-handles-convs}

Convolutional layers do not fit cleanly into this dense picture. The trainable parameters and their gradient updates are finite-support 4D kernels: $G_K \in \mathcal{M}_{\mathrm{kernel}}$, where $\mathcal{M}_{\mathrm{kernel}}=\mathbb{R}^{c_{\mathrm{out}} \times c_{\mathrm{in}} \times k_h \times k_w}$

Unfortunately, materializing the equivalent linear transformation to the convolution is infeasible at scale. Therefore, current implementations of Muon sidestep this by using a structural shortcut: they treat the 4D kernel as a 2D matrix by flattening its spatial dimensions, apply the exact same dense Newton-Schulz orthogonalization routine, and unflatten the result:
Denoting the flattening and unflattening operators as $\operatorname{flat}: \mathbb{R}^{c_{\mathrm{out}} \times c_{\mathrm{in}} \times k_h \times k_w} \to \mathbb{R}^{c_{\mathrm{out}} \times (c_{\mathrm{in}} k_h k_w)}$ and $\operatorname{unflat}: \mathbb{R}^{c_{\mathrm{out}} \times (c_{\mathrm{in}} k_h k_w)} \to \mathbb{R}^{c_{\mathrm{out}} \times c_{\mathrm{in}} \times k_h \times k_w}$, the Reshape-Kernel Orthogonalization (RKO) update is:
\begin{equation}
    \operatorname{RKO}(G_K) = \operatorname{unflat}\bigl(\operatorname{PF}(\operatorname{flat}(G_K))\bigr)
\end{equation}
This formulation effectively enforces an orthogonality constraint on the flattened kernel matrix. We can formalize the surrogate objective implicitly minimized by RKO as $\mathcal{L}_{\mathrm{flat}}$:
\begin{equation}
    \mathcal{L}_{\mathrm{flat}}(\Delta_K) = \| \operatorname{flat}(\Delta_K) \operatorname{flat}(\Delta_K)^\top - I \|_F^2
    \label{eq:lflat}
\end{equation}
By driving $\mathcal{L}_{\mathrm{flat}}$ to zero, RKO makes the flattened matrix perfectly semi-orthogonal. While computationally convenient, $\operatorname{flat}(G_K)$
ignores the true spatial and channel interaction of the convolution, optimizing a mismatched geometric surrogate. For example,
it
cannot distinguish a $3 \times 3$ convolution with $c_{\mathrm{in}}=8$ and $c_{\mathrm{out}}=8$ from a pointwise $1 \times 1$ convolution with $c_{\mathrm{in}}=72$ and $c_{\mathrm{out}}=8$, even though those objects act fundamentally differently on activations. Notably, \citet{wang2020orthogonal, li2019preventing} show that
RKO
does not lead to an orthogonal linear operator, as illustrated in the upper part of \cref{fig:problem-framing}.

\section{Our Correction: The Convolutional Polar Factor}
\label{sec:conv-polar-factor}

To correctly adapt Muon to convolutions, we must work in the space in which the layer actually acts.
For a given input activation map, convolution by a kernel $K$ is formally equivalent to a structured linear operator $\mathcal{T}(K) \in \mathbb{R}^{(c_{\mathrm{out}} h_{\mathrm{out}} w_{\mathrm{out}}) \times (c_{\mathrm{in}} h_{\mathrm{in}} w_{\mathrm{in}})}$, typically a doubly circulant or Toeplitz matrix depending on boundary conditions \citep{achour2022existence}.

The equivalent Muon update for a convolution layer would therefore be the polar factor (\Cref{eq:dense-polar-projection}) of the induced operator: $\operatorname{PF}(\mathcal{T}(G_K))\in \mathbb{R}^{(c_{\mathrm{out}} h_{\mathrm{out}} w_{\mathrm{out}}) \times (c_{\mathrm{in}} h_{\mathrm{in}} w_{\mathrm{in}})}$.
As settled for dense matrices,
the polar factor of this polar factor
$\operatorname{PF}(\mathcal{T}(G_K))$ could be approximated by iterative algorithms (\Cref{eq:dense-newton-schulz}). But, first applying Newton-Schultz on large matrices like $\mathcal{T}(G_K)$ would be costly, and the resulting polar factor generally loses its structural sparsity: $\operatorname{PF}(\mathcal{T}(G_K))$ cannot be converted back into a valid
$\mathcal{M}_{\mathrm{kernel}}$ update
(lower part of \cref{fig:problem-framing}). 
Therefore, a principled definition of the convolutional polar factor requires an extra constraint: the solution operator must remain in the subspace of
fixed-support convolutions $\mathcal{M}_{\mathrm{kernel}}$:
\vspace{-0.1mm}
\begin{equation}
    \convPF(G_K) = \underset{
    \substack{
    \Delta_K\in \mathcal{M}_{\mathrm{kernel}}\\
    \mathcal{T}(\Delta_K)\in \stief}
    }{\operatorname{argmin}} \quad \mathcal{L}_{\mathrm{frob}}(\mathcal{T}(G_K),\mathcal{T}(\Delta_K)).
    \label{eq:conv_op_lmo}
\end{equation}

Closeness to the Stiefel manifold is assessed through the semi-orthogonality residual
$\|\mathcal{T}(\Delta_K) \mathcal{T}(\Delta_K)^\top - I\|^2$. \citet{wang2020orthogonal, qi2020deep, achour2022existence} have shown that this analogously translates to a spatial self-convolution identity on the underlying kernel. We denote this kernel-level orthogonality distance as $\mathcal{L}_{\mathrm{orth}}$:
\begin{align}
    \mathcal{L}_{\mathrm{orth}}(\Delta_K)
    =&
    \|\operatorname{conv}
    \bigl(
    \Delta_K, \Delta_K,
    \mathrm{padding}=P,
    \mathrm{stride}=s
    \bigr) - I_{r0} \|_F^2
    \\
    =&
    \frac{1}{h_{\mathrm{out}} w_{\mathrm{out}}} \|\mathcal{T}(\Delta_K)\mathcal{T}(\Delta_K)^\top - I\|_F^2,
    \label{eq:lorth-row}
\end{align}
where $I_{r0}$ is the identity impulse kernel (zero everywhere except at the central position where it is the identity matrix over channels) and $h_{\mathrm{out}} w_{\mathrm{out}}$ is the number of spatial locations. The \cref{eq:lorth-row}, formally demonstrated in \citet{achour2022existence}, establishes the exact equivalence\footnote{For simplicity, we assume the row-orthogonal (wide) case $c_{\mathrm{out}} \le c_{\mathrm{in}} s^2$, analogous to our dense formulation in \cref{subsec:muon-background}. In the column-orthogonal case, $\mathcal{L}_{\mathrm{orth}}(\Delta_K)$ cannot reach $0$ (its global minimum is bounded by $c_{\mathrm{out}} - c_{\mathrm{in}} s^2$), but its minimizer nonetheless yields a semi-orthogonal convolutional operator.}. The $\convPF(G_K)$ can thus be expressed using the $\mathcal{L}_{\mathrm{orth}}$:
\begin{equation}
    \convPF(G_K) = \underset{
    \substack{
    \Delta_K\in \mathcal{M}_{\mathrm{kernel}}\\
    \mathcal{L}_{\mathrm{orth}}(\Delta_K) = 0}
    }{\operatorname{argmin}} \quad \mathcal{L}_{\mathrm{frob}}(\mathcal{T}(G_K),\mathcal{T}(\Delta_K)).
    \label{eq:conv_op_lmo2}
\end{equation}

Thus the convolutional polar factor is the equivalent for convolution operators of the target update of Muon for dense matrices. However, neither the SVD decomposition of the operator $\mathcal{T}(G_K)$, nor the RKO heuristic of Muon is able to compute or even approximate $\convPF(G_K)$.

\section{Convolutional Newton-Schulz: A Convolution-Aware Alternative}
\label{sec:conv-newton-schulz}

\begin{figure}[t]
    \begin{minipage}{0.44\textwidth}
        \begin{algorithm}[H]
            \caption{\textbf{Conv-NS}}
            \label{alg:convNS_k}
            \begin{algorithmic}[1]
                \Require Gradient kernel $G_K$, iterations $T$, step size $\beta$
                \State $\Delta_K^{(0)} \gets \operatorname{Normalize}(G_K)$
                \For{$t=0,\ldots,T-1$}
                    \State $C^{(t)} \gets \Delta_K^{(t)}
                    \circledast
                    (\Delta_K^{(t)})^\top
                    \circledast
                    \Delta_K^{(t)}$
                    \State $\widetilde{\Delta}_K^{(t)}
                    \gets
                    (1+\beta)\Delta_K^{(t)}-\beta C^{(t)}$
                    \State $\Delta_K^{(t+1)}
                    \gets
                    \alpha\Pi_{\mathrm{supp}}\Bigl(\widetilde{\Delta}_K^{(t)}\Bigr)$
                \EndFor
                \State \Return $\Delta_K^{(T)}$
            \end{algorithmic}
        \end{algorithm}

        \caption{\textbf{Conv-NS orthogonalizes convolutional updates while preserving kernel support.} Each iteration applies a Newton--Schulz step through block convolutions, then crops the result to the original kernel support.}
    \end{minipage}
    \hfill
    \begin{minipage}{0.53\textwidth}

        \centering
        \includegraphics[width=\linewidth]{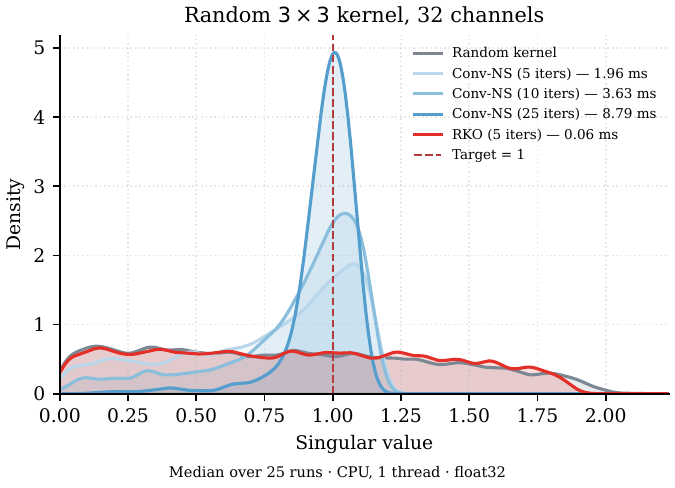}
        \caption{\textbf{Conv-NS converges to convolutional orthogonality.} Increasing Conv-NS iterations concentrates the spectrum of singular values around one, whereas RKO, used in standard Muon implementations, leaves it dispersed. All methods start from the same random $3\times3$ kernel.}
        \label{alg:convNS_svd}
        \label{fig:MuonSVD}
    \end{minipage}
\end{figure}

We
introduce Convolutional Newton-Schulz (Conv-NS), an efficient convolution-aware iterative algorithm that converges to $\convPF(G_K)$. Conv-NS serves a dual purpose: it provides a convenient method to orthogonalize the convolution gradient matrix that can be used as a gradient update,
evaluated in \cref{subsec:cifar10-imagenet}, and a principled reference for measuring how RKO differs from the true convolutional polar factor. We exploit this second role in \cref{subsec:why-reshape-works} to characterize RKO on real training gradients.
The proposed algorithm enables to compute the ground-truth comparison that the existing literature lacks.

We propose an iterative projected gradient descent scheme to jointly  drive to  orthogonality and constrain to the $\mathcal{M}_{\mathrm{kernel}}$ manifold. After each iteration $t$ of the order-3 Newton-Schulz (also known as Bj\"orck) algorithm \citep{bjorck1971iterative}, a structural projection $\Pi_{\mathrm{supp}}$ is applied to force the kernel back onto $\mathcal{M}_{\mathrm{kernel}}$:
\begin{align}
    \mathcal{T}(\widetilde{\Delta}_K^{(t)}) & \gets (1+\beta)\mathcal{T}
    (\Delta_K^{(t)}) - \beta \mathcal{T}(\Delta_K^{(t)})\mathcal{T}(\Delta_K^{(t)})^\top\mathcal{T}(\Delta_K^{(t)}) \\
    \Delta_K^{(t+1)} &
    \gets
    \alpha\Pi_{\mathrm{supp}}\Bigl(\widetilde{\Delta}_K^{(t)}\Bigr),
    \label{eq:conv-ns-matrix}
\end{align}

where $\beta_t$ is the orthogonalization step size, and $\alpha$ a normalization factor to avoid expansion (justification of this parameter is given in \cref{sec:ap-convNS-impl}). Without the projection (\Cref{eq:conv-ns-matrix}), the kernel size will grow at each iteration, and would converge to $\operatorname{PF}(\mathcal{T}(G_K))$ as described in \Cref{sec:conv-polar-factor}.

In practice, these iterations would require huge matrix multiplications with high computational redundancy due to the doubly circulant or Toeplitz structure of $\mathcal{T}(\Delta_K^{(t)})$.
We instead propose to use block-convolution to perform the NS-3 scheme directly on the kernels $\Delta_K^{(t)}$.
Following \citet{li2019preventing,boissinadaptive}, the \emph{block convolution} operator $\circledast$ combines two kernels into a single kernel representing their composition (without materializing their full convolution matrices). For $A \in \mathbb{R}^{c_{\mathrm{int}} \times c_{\mathrm{in}} \times k^A_h \times k^A_w}$ and $B \in \mathbb{R}^{c_{\mathrm{out}} \times c_{\mathrm{int}} \times k^B_h \times k^B_w}$, their block convolution is defined with respect to the induced linear operators as:
\begin{equation}
    \mathcal{T}(B \circledast A) = \mathcal{T}(B)\mathcal{T}(A)
    \quad \text{with} \quad
    B \circledast A \in \mathbb{R}^{c_{\mathrm{out}} \times c_{\mathrm{in}} \times (k^A_h+k^B_h-1) \times (k^A_w+k^B_w-1)}.
\end{equation}
We also use the definition of the transposed convolution kernel $\mathcal{T}(\Delta_K^\top) = \mathcal{T}(\Delta_K)^\top$, which can be computed efficiently in kernel space.

The complete Convolutional Newton-Schultz algorithm is provided as a kernel formulation in \Cref{alg:convNS_k}. Lines 3 and 4 are equivalent to the order-3 Newton-Schulz iteration. Importantly, the block convolution expands the spatial support of the operands, so $C^{(t)}$ kernel size is generally larger than $\Delta_K^{(t)}$ one.  The  projection $\Pi_{\mathrm{supp}}$ crops the expanded kernel by dropping the values outside the original $k_h \times k_w$ support, to converge to $\operatorname{convPF}(\mathcal{T}(G_K))$.

In \Cref{sec:appendix-convergence}, we provide the full proofs, together with an explicit gradient formula, showing that the Conv-NS algorithm converges to a minimum of $\mathcal{L}_{\mathrm{orth}}$. \cref{sec:ap-convNS-impl} provide more details about the exact implementation and how strided and grouped convolutions are covered. In \Cref{app:grabbing}, we also propose convolutional adaptations of the main grafting factors commonly applied to Muon updates.

\begin{figure}[ht]
    \centering
    \includegraphics[width=\columnwidth]{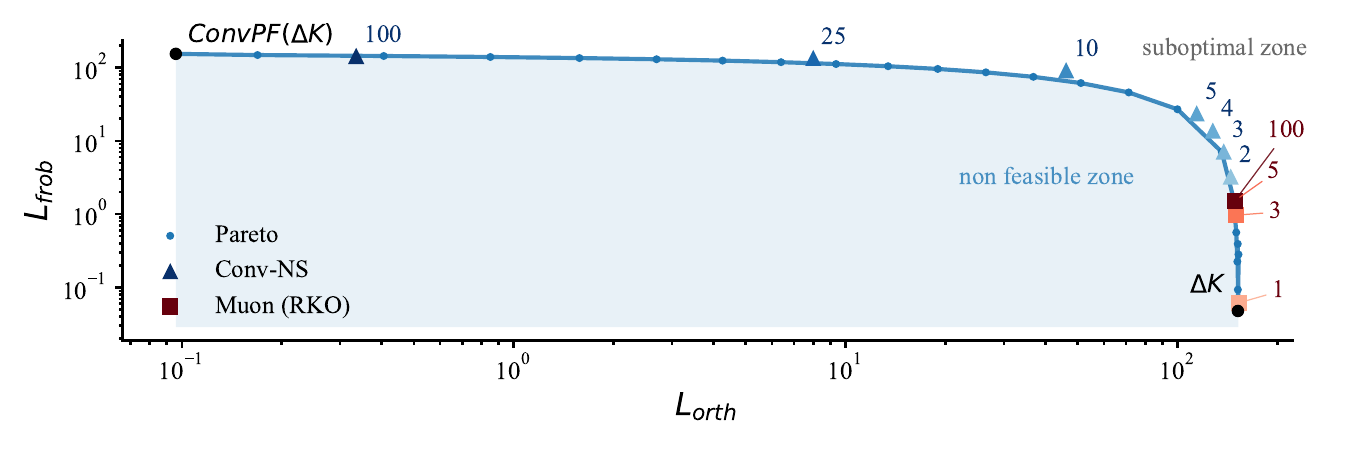}
    \caption{\textbf{Conv-NS approaches convolutional orthogonality with near-minimal gradient modification.} Conv-NS follows the numerically estimated orthogonality--fidelity frontier: stronger orthogonality requires progressively larger departures from the original gradient along this frontier. RKO remains near the low-modification end, with higher orthogonality error. Metrics are averaged across convolutional layers and three stages of CIFAR-10 training.}
    \label{fig:grad-pareto-real}
\end{figure}

\paragraph{Computing a ground truth for ConvPF.}
Since $\convPF$ minimizes gradient modification under convolutional orthogonality (\cref{eq:conv_op_lmo2}), we construct a numerical reference by relaxing this constraint into a penalty:
\begin{equation}
    \underset{\Delta_K\in\mathcal{M}_{\mathrm{kernel}}}{\min}
    \quad \lambda\mathcal{L}_{\mathrm{orth}}(\Delta_K)
    +(1-\lambda)\mathcal{L}_{\mathrm{frob}}(\Delta_K,G_K),
    \qquad 0\leq\lambda<1.
    \label{eq:convpf-penalty}
\end{equation}
Here, $G_K$ is the optimizer momentum rescaled to unit operator norm to separate orthogonalization-induced modification from initial normalization.
Varying $\lambda$ trades fidelity to the original gradient against convolutional orthogonality. Since this objective is nonconvex, we estimate the orthogonality--fidelity frontier using Adam with simulated annealing noise injection for each weighting, retaining the best result across multiple trials (\cref{subsec:pareto-analysis}). This provides a numerical reference for evaluating Conv-NS, rather than a certificate of global optimality.

\paragraph{Empirical validation:}\Cref{fig:MuonSVD} compares the singular-value spectra of the RKO Muon update and Conv-NS after different numbers of iterations. It shows that Conv-NS progressively approaches an orthogonal convolution operator.
We compare Conv-NS against this numerical reference on synthetic kernels (\Cref{fig:pareto-baselines}) and real training gradients (\Cref{fig:grad-pareto-real}). For the latter, we collect snapshots at 1\%, 50\%, and 99\% of a CIFAR-10 training run using the architecture of \cref{subsec:cifar10-imagenet}. We evaluate Conv-NS at iteration counts $\{2,3,4,5,10,25,100\}$ against the per-layer frontiers and average metrics across layers and stages. Conv-NS remains close to the estimated frontier, approaching orthogonality with near-minimal gradient modification among the solutions found numerically. Increasing the iteration count moves along this trade-off toward stronger orthogonality and larger gradient modification. Per-stage results are reported in \cref{sec:appendix-per-stage-pareto}.

\section{Experiments: CIFAR-10 and ImageNet classification}
\label{subsec:cifar10-imagenet}
In the following, we compare Conv-NS Muon against the standard reshape-based Muon on several image-classification benchmarks with convolutional architectures.

\paragraph{Experimental setup.}
We evaluate Conv-NS in two classification settings that stress different aspects of convolutional optimization.
First, CIFAR-10 AirBench94~\citep{jordan_airbench_2024} is a speed-optimized benchmark tailored to reach 94\% accuracy in less than 4 seconds on a single RTX4090. Second, a standard ResNet-50 with a \texttt{timm}-style recipe shortened to 30 epochs is trained on ImageNet-1k~\citep{deng2009imagenet}. In addition to these standard short-training budgets, we evaluate an extreme 1-epoch setting on both datasets. This 1-epoch regime is specifically designed to isolate early optimization efficiency and highlight differences before network capacity becomes the limiting factor.
On CIFAR-10, reshape-based Muon and SGD use the original recipe's
heavily optimized settings, while Conv-NS is tuned with a 1000-configuration budget.
On ImageNet, reshape-based Muon and Conv-NS receive equal 48-configuration grid-search budgets.
Full experimental details are reported in \cref{sec:appendix-classification-details}; results are summarized in \cref{tab:cifar10-imagenet-results}.

\paragraph{Both Muon-Reshape and Muon-ConvNS are efficient.}
The results highlight that Muon variants drastically outperform SGD in short training settings. This is particularly prevalent in the extreme 1-epoch ImageNet setting, where both Muon (RKO) and the proposed Muon-Conv-NS  methods achieve an accuracy above $52\%$, showcasing a massive gap compared to SGD's $30.47\%$. This is, to our knowledge, the best result that can be achieved without techniques such as replay buffers~\citep{hu2021one}. A deeper comparison with these techniques can be found in \cref{app:one-pass-imagenet}. However, as training gets longer, performance becomes increasingly capped by the network's capacity rather than the optimizer's speed. This dynamic explains why the performance gap between Muon and SGD narrows in the 8-epoch and 30-epoch settings, and it also contextualizes why SGD remains highly competitive for very long training schedules, such as the 600-epoch runs used in modern ResNet baselines~\citep{wightman2021resnet}, which prioritize generalization.

\paragraph{Conv-NS closes the theory gap with competitive efficiency.}
Conv-NS successfully fixes the theoretical operator mismatch while achieving performance that is on par with the reshape-based heuristic. On CIFAR-10, Conv-NS matches the accuracy of reshape-based Muon ($93.98\%$ vs.\ $94.02\%$). On ImageNet, it achieves a competitive $74.53\%$ accuracy, only slightly below reshape-based Muon's $74.82\%$. In terms of runtime, the principled approach is highly practical: it introduces only a marginal overhead on the small CIFAR-10 network and actually runs slightly faster than the reshape baseline on the larger ResNet-50 architecture. This demonstrates that closing the theory-practice gap does not require sacrificing computational efficiency. Additional comparisons with other unpublished baselines are reported in \cref{subsec:cifar10-community-baselines}.

\section{What should convolutional orthogonalization optimize?}
\label{subsec:why-reshape-works}
The near-parity with Conv-NS observed on both CIFAR-10 and ImageNet raises a central question:
\textit{Why does improving operator orthogonalization in Muon for convolution updates fail to outperform the existing heuristic?}

\begin{table}[t]
    \centering
    \small
    \begin{tabular}{llccc}
        \toprule
        \textbf{Dataset (Budget)} & \textbf{Optimizer} & \textbf{Val.\ Acc.\ (\%)} & \textbf{Train Acc.\ (\%)} & \textbf{Runtime} \\
        \midrule
        \textbf{CIFAR-10 (8 epochs)} & Muon - reshape & $94.02 \pm 0.14$ & $97.66 \pm 0.31$ & $3.074 \pm 0.005$s \\
        8 layers CNN, $bs=2000$ & Muon - Conv-NS & $93.98 \pm 0.14$ & $97.17 \pm 0.38$ & $3.608 \pm 0.018$s \\
        1xRTX4090 & SGD            & $91.25 \pm 0.09$ & $93.00 \pm 0.35$ & $2.975 \pm 0.009$s \\
        \midrule
        \textbf{CIFAR-10 (1 epoch)} & Muon - reshape & $87.18 \pm 0.002$ & $86.46 \pm 0.01$ & $0.509 \pm 0.003$s \\
        8 layers CNN, $bs=250$ & Muon - Conv-NS & $87.75 \pm 0.002$ & $85.44 \pm 0.02 $ & $0.559 \pm 0.003$s \\
        1xRTX4090 & SGD            & $84.99 \pm 0.002 $ & $ 83.52 \pm 0.02 $ & $0.496 \pm 0.004$s \\
        \midrule
        \textbf{ImageNet (30 epochs)} & Muon - reshape & $74.82$ & $60.51$ & $1$\,h $54$m $06$s $\pm 12$m $51$s \\
        ResNet-50, $bs=2048$& Muon - Conv-NS & $74.53$ & $58.70$ & $1$\,h $49$m $20$s $\pm 1$m $42$s \\
        4xH100& SGD            & $70.25$ & $48.03$ & $1$\,h $33$m $56$s $\pm 8$m $07$s \\
        \midrule
        \textbf{ImageNet (1 epoch)}& Muon - reshape & $54.05$ & $28.10$ & $14$m $57$s $\pm 4$m $57$s \\
        ResNet-50, $bs=512$ & Muon - Conv-NS & $52.95$ & $26.74$ & $13$m $13$s $\pm 44$s \\
        1xH100 & SGD            & $30.47$ & $10.10$ & $13$m $00$s $\pm 3$m $46$s \\
        \bottomrule
    \end{tabular}
    \caption{\textbf{CIFAR-10 and ImageNet classification results.} Conv-NS essentially matches reshape-based Muon on CIFAR-10 AirBench94 and remains competitive on ImageNet-1k, while reshape-based Muon keeps a slight overall edge. Train and validation accuracies are shown alongside runtimes for both the full budget (8/30 epochs) and the 1-epoch setting.}
    \label{tab:cifar10-imagenet-results}
\end{table}

\paragraph{More accurate orthogonalization can hurt training.}
We sweep the number of Conv-NS iterations in the one-epoch CIFAR-10 setting of \cref{subsec:cifar10-imagenet}. Under training-loss selection, validation accuracy peaks at 10 iterations ($87.56\%$), then decreases at 25 and 100 iterations ($86.42\%$ and $85.13\%$); training loss also worsens (\Cref{tab:1-epoch-sweep}). The same qualitative trend holds under validation- and training-accuracy-based selection (\Cref{tab:1-epoch-selection-criteria}). In the eight-epoch setting, \Cref{fig:ablation-overconstraint} shows that increasing the iteration count substantially reduces $\mathcal{L}_{\mathrm{orth}}$ throughout training, yet the 100-iteration setting generally learns more slowly than intermediate iteration counts. Thus, more accurate convolutional orthogonalization does not consistently improve training performance. Learning-rate sweeps for each iteration count are provided in \cref{sec:appendix-8-epoch}.
\begin{table}[h]
    \centering
    \small
    \setlength{\tabcolsep}{5pt}
    \begin{tabular}{lccccc}
        \toprule
        Iterations (Conv-NS) & 1 & 5 & 10 & 25 & 100 \\
        \midrule
        Val. Acc. (\%) & 84.93 & 86.35 & \textbf{87.56} & 86.42 & 85.13 \\
        Train Acc. (\%) & 86.32 & 86.72 & \textbf{87.44} & 86.52 & 86.32 \\
        Train Loss & 1.1693 & 1.1519 & \textbf{1.1352} & 1.1571 & 1.1821 \\
        Muon LR & 0.579 & 0.513 & 0.487 & 0.200 & 0.094 \\
        Muon Momentum & 0.799 & 0.724 & 0.250 & 0.406 & 0.706 \\
        \bottomrule
    \end{tabular}
    \vspace{0.25em}
    \caption{\textbf{One-epoch CIFAR-10 sweep results.} We report the single top configuration per sweep under training-loss selection; each configuration line is already averaged over repeated runs with the same hyperparameters. The best one-epoch operating point occurs at 10 Conv-NS iterations.}
    \label{tab:1-epoch-sweep}
\end{table}

\paragraph{Lower orthogonality errors do not guarantee better training.}
We train the same CIFAR-10 network for eight epochs with RKO and Conv-NS using $n\in\{1,5,10,25,100\}$, with five runs per setting. Conv-NS settings share a learning rate of $0.52$, momentum $0.6$, and damping $0.99$; RKO uses learning rate $0.24$ and the same momentum. We track training and validation accuracy alongside $\mathcal{L}_{\mathrm{orth}}$, $\mathcal{L}_{\mathrm{flat}}$, and $\mathcal{L}_{\mathrm{frob}}$, averaged across convolutional layers.
RKO's performance raises the question of whether its flattened-kernel objective $\mathcal{L}_{\mathrm{flat}}$ is better suited to training than $\mathcal{L}_{\mathrm{orth}}$. However, \Cref{fig:ablation-overconstraint} shows that high-iteration Conv-NS reduces both objectives and can achieve lower $\mathcal{L}_{\mathrm{flat}}$ than RKO without improving validation accuracy. Thus, lower values of either orthogonality objective alone do not explain the performance ordering. Full residual comparisons are reported in \cref{sec:appendix-hierarchical}. Since Conv-NS empirically approaches convolutional orthogonality with near-minimal gradient modification (\Cref{sec:conv-newton-schulz}), these results motivate examining whether the $\convPF$ target itself imposes unnecessary constraints.

\paragraph{Convolutional structure limits the LMO justification for polar factorization.}
Several works in the literature have proposed explanations for the empirical performance of Muon and gradient orthogonalization. One of the most prominent interpretations relates Muon to approximately solving a Linear Minimization Oracle (LMO) over a spectral-norm ball of radius $\eta$\footnote{In the following, $\eta$ may be omitted and replaced by the unit ball for simplicity}~\citep{pethick_training_2025}:
\begin{equation}
    \Delta_W^\star =
    \underset{\|\Delta_W\|_{\mathrm{op}} \le \eta}{\arg\min}
    \langle G, \Delta_W \rangle_F .
    \label{eq:dense-lmo}
\end{equation}
This formulation measures the size of the update through the operator norm (spectral norm), thereby constraining the worst-case change in the layer output \citep{bernstein_modular_2025}. By duality between the spectral and nuclear norms, the closed-form solution is proportional to the semi-orthogonal polar factor of the gradient \citep{bernstein_old_2024,newhouse_understanding_2025}. For completeness, we provide the proof of the optimality of $\operatorname{PF}(G)$ in \cref{app:proofLMO}.

For convolutional layers, the LMO must instead be formulated in the operator space in which the layer acts. The minimization is therefore restricted to structured linear operators induced by fixed-support kernels, with $G_K,\Delta_K\in\mathcal{M}_{\mathrm{kernel}}$:
\begin{equation}
    \Delta_K^\star =
    \underset{
    \substack{
    \Delta_K\in \mathcal{M}_{\mathrm{kernel}}\\
    \|\mathcal{T}(\Delta_K)\|_{\mathrm{op}} \le 1
    }
    }{\arg\min}
    \langle \mathcal{T}(G_K),\mathcal{T}(\Delta_K)\rangle_F .
    \label{eq:conv-lmo}
\end{equation}

The connection between the convolutional LMO in \cref{eq:conv-lmo} and the convolutional polar factor $\convPF(G_K)$ follows the same proof as for dense matrices (\Cref{app:proofLMO}). As a first indication, $\convPF(G_K)$ satisfies the unit-ball constraint $\|\mathcal{T}(\Delta)\|_{\mathrm{op}}\leq 1$ by construction, unlike the RKO solution (\cref{subsec:how-muon-handles-convs}). Moreover, decomposing $\mathcal{L}_{\mathrm{frob}}(\mathcal{T}(G_K),\mathcal{T}(\Delta_K))$ shows that $\convPF(G_K)$ solves the LMO over the Stiefel-constrained feasible set (\Cref{app:proofconvLMO}):

\begin{equation}
    \label{eq:convPFargmin}
    -\convPF(G_K) \in \underset{
    \substack{
    \Delta_K\in \mathcal{M}_{\mathrm{kernel}}\\
    \mathcal{T}(\Delta_K)\in \stief}
    }
    {\operatorname{argmin}} \quad \langle \mathcal{T}(G_K), \mathcal{T}(\Delta_K) \rangle_F \quad
\end{equation}
The main difference between \Cref{eq:conv-lmo} and \Cref{eq:convPFargmin} lies in the minimization set which are respectively $\|\mathcal{T}(\Delta)\|_{\mathrm{op}}\leq 1$ and $\mathcal{T}(\Delta_K)\in \stief$.

Finally, in the dense case, polar-factor optimality follows from $\operatorname{CH}(\stief)$ (the Convex Hull of $\stief$) being the operator-norm unit ball. Does this remain true under the kernel-size constraint in \cref{eq:convPFargmin}?
Indeed, the optimum $\mathcal{T}(\convPF(G_K))$ lies both in $\mathcal{T}(\mathcal{M}_{\mathrm{kernel}})= \left\{ \mathcal{T}(\Delta), \Delta\in \mathcal{M}_{\mathrm{kernel}}\right\}$, and on the Stiefel manifold $\stief$.
Therefore, the convex hull of their intersection is contained in the intersection of their convex hulls:
\begin{equation}
    \operatorname{CH}(\mathcal{T}(\mathcal{M}_{\mathrm{kernel}}) \cap \stief) \subseteq \mathcal{T}(\mathcal{M}_{\mathrm{kernel}}) \cap \operatorname{CH}(\stief)=\left\{ \Delta \in \mathcal{M}_{\mathrm{kernel}} \text{ such that } \|\mathcal{T}(\Delta)\|_{op}\leq 1\right\}
\end{equation}
The right-hand side is exactly the feasible set of the convolutional LMO in
\cref{eq:conv-lmo}.
Since the inclusion is generally strict,
$\convPF(G_K)$ \textit{is optimal over a more constrained set than the operator-norm unit ball restricted to} $\mathcal{T}(\mathcal{M}_{\mathrm{kernel}})$.

This distinction offers a possible explanation for the previous results: approaching $\convPF$ more accurately enforces orthogonality constraints beyond what the convolutional LMO requires. The accompanying increase in gradient modification may therefore hinder early optimization, even as both $\mathcal{L}_{\mathrm{orth}}$ and $\mathcal{L}_{\mathrm{flat}}$ decrease (\Cref{tab:1-epoch-sweep,fig:ablation-overconstraint}).


Both approaches therefore have limitations: reshape-based Muon does not enforce the convolutional operator-norm constraint (under constraint), while Conv-NS targets the stronger orthogonality constraint (over constraint). Neither is guaranteed to solve the convolutional LMO; a principled extension of Muon to convolutions requires more than computing the convolutional polar factor.

\begin{figure*}[t]
    \centering
    \includegraphics[width=\textwidth]{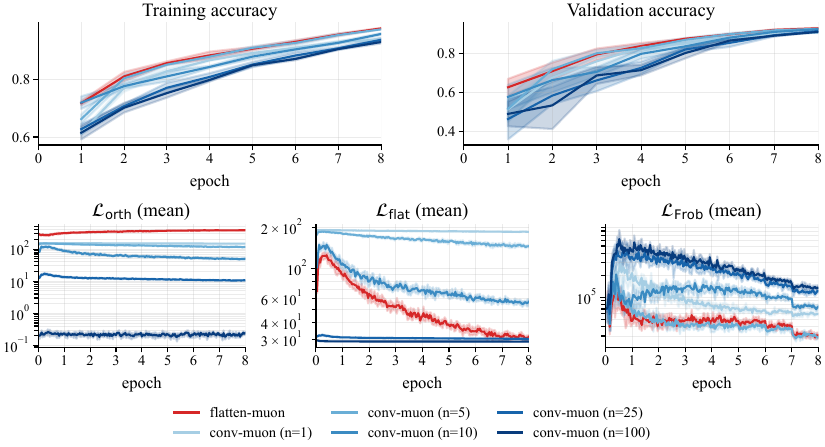}
    \caption{\textbf{Better convolutional and flattened orthogonality does not guarantee better training.} Increasing Conv-NS iterations lowers both $\mathcal{L}_{\mathrm{orth}}$ and $\mathcal{L}_{\mathrm{flat}}$, yet high iteration counts slow early learning and increase gradient modification. Curves show means over five eight-epoch CIFAR-10 runs; shaded bands indicate one standard deviation.}
    \label{fig:ablation-overconstraint}
\end{figure*}

\section{Conclusion}
\label{sec:conclusion}



In this work, we revisited Muon for convolutional layers from the perspective of the induced convolution operator. We showed that the standard reshape-based implementation does not, in general, orthogonalize this operator.

We introduced Convolutional Newton-Schulz (Conv-NS), a convolution-aware procedure that preserves kernel support and converges toward the convolutional polar factor. 
Despite enforcing better operator-level orthogonality, Conv-NS only achieves performance comparable to reshape-based Muon on CIFAR-10 and ImageNet. 
Moreover, increasing the number of Conv-NS iterations could degrade performance, suggesting that exact convolutional orthogonality may overconstrain the update.

We finally relate this observation to the LMO interpretation of Muon: the convolutional polar factor is optimal under the Stiefel constraint, 
but not necessarily the minimizer of the full operator-norm LMO.
This suggests that adapting Muon to convolutions requires balancing operator-level control with sufficient freedom in the update.

Finally, future work could explore more various training tasks like generative modeling or self-supervised representation learning with Conv-NS; along with its extensions for non-2D convolutional layers. Indeed, real training gradient analysis on classification tasks seems to indicate that reshape-based Muon empirically finds a compromise between gradient fidelity and orthogonality, this is not the case on synthetic gradients of different distributions (see~\cref{fig:pareto-baselines}). 
This distribution-dependent optimality indicates that reshape-based Muon could be suboptimal on different training tasks with different gradient distributions.

\ifshowauthors
\section*{Acknowledgments}

This work was carried out within the DEEL project,\footnote{\url{https://www.deel.ai/}} which is part of IRT Saint Exupéry and the ANITI AI cluster. The authors acknowledge the financial support from DEEL's Industrial and Academic Members and the France 2030 program - Grant agreements n°ANR-10-AIRT-01 and n°ANR-23-IACL-0002.

\noindent This work was granted access to the HPC resources of IDRIS under the allocation 2025-AD011016381 made by GENCI. 
\fi

\ifarxiv
\else
\subsection*{AI use statement}

In this work, we used generative AI tools for coding assistance in this project, we limited the use of this tooling to boilerplate functions and utilities.
Every function or class written with AI assistance was rigorously checked and tested. 
The main optimizer class is nearly entirely human written except for checked boilerplate code.
Additionally, AI tools were used for polishing grammatical errors in the writing of the main paper. 
AI tooling was used to check the soundness of the existing human written proofs that were checked previously by all authors. 
AI tooling was also used to improve the visual formatting of pre-existing human written plotting scripts, these scripts \textit{only} changed the formatting of said plots.
All concepts and novel ideas in this paper are solely the authors’ work. No AI system was used for ideation or concept generation.
AI systems were tasked to identify potential gaps in the authors’ manually conducted related-work survey.
The additional works it found were all manually checked for existence and relevance to the paper's message and studies. 
We take responsibility for the final content of this work,
including text, claims or artifacts produced with the aid of generative AI.

\subsection*{Reproducibility statement}

Along with this work, we provide access to the code used to run the experiments of the paper.
More specific details are given in~\cref{sec:appendix-classification-details}.
\fi


\bibliographystyle{iclr2027_conference}
\bibliography{files/11_references}

\clearpage
\appendix

\section{Related Work}
\label{sec:related}

\subsection{Muon and Matrix-Whitening Optimization}
The development of Linear Minimization Oracle (LMO) and matrix-whitening optimizers, such as Muon~\citep{jordan_muon_2024,bernstein_deriving_2025} and SOAP~\citep{vyas2025soap,wen_fantastic_2025}, aims to improve upon standard element-wise algorithms like AdamW. \citet{frans_what_2025} argue that the gains of matrix-whitening optimizers are not explained solely by spectral normalization: in their Transformer language-modeling experiments, variance-adapted variants such as Adam, SOAP, and AdaMuon outperform sign-descent counterparts such as Signum, SPlus, and Muon. Applying such optimizers alters the implicit bias of the network, fundamentally shifting convergence toward max-margin solutions with respect to the spectral norm rather than prioritizing simple dominant components~\citep{dragutinovic_use_2026,vasudeva_how_2025}. While this structural bias often benefits generalization on complex distributions, controlling the resulting weight drift remains challenging, leading to variants like Spectral Sphere Optimizer (SSO)~\citep{xie_controlled_2026} which enforces strict module-wise spectral properties. Despite reports of substantial compute-time frontier breakthroughs and improved large-batch efficiency~\citep{ai_practical_2025}, recent rigorous scaling laws show that the speedup relative to AdamW diminishes at larger parameter scales depending on careful hyperparameter tuning~\citep{wen_fantastic_2025}. Furthermore, theoretical stepsizes are often bridged to empirical reality by accounting for inexact matrix updates~\citep{shulgin_beyond_2026} and layer-wise geometry~\citep{riabinin_gluon_2025}.

\subsection{Community Explorations for Convolutional Orthogonalization}
The structural gap between 4D kernels and 2D spectral optimizers has spurred recent community efforts to adapt Muon for CNNs, notably Freq-Muon~\citep{kim2026frequency,cesista2026freqmuon}, which orthogonalizes per-frequency channel-mixing matrices via the DFT, and Matrixizer~\citep{opooladz_matrixizer}, which searches for an optimal 2D factorization of the gradient tensor before applying Newton--Schulz.
These implementations highlight an active demand for a CNN-compatible Muon, yet no peer-reviewed work has formalized these approaches or provided rigorous evaluations.
We describe both methods in detail and report appendix-only comparisons in \cref{sec:appendix-community-baselines}.

\section{Implementation of Conv-NS}
\label{sec:ap-convNS-impl}

\begin{algorithm}[h]
\caption{Convolutional Newton--Schulz (Conv-NS)}
\label{alg:conv-ns}
\begin{algorithmic}[1]
\Require Kernel $K$, stride $s=(s_h,s_w)$, groups $g$
\Require $N=12$, $\beta=0.5$, $\alpha=0.91$, $\varepsilon=10^{-2}$
\State $(c_o,c_i/g,k_h,k_w)\gets\operatorname{shape}(K)$
    
\State $K\gets\Call{SpectralRescale}{K}$
    \Comment{Three Gram iterations}
\State $p\gets
    \bigl(s_h\lfloor(k_h-1)/s_h\rfloor,\,
          s_w\lfloor(k_w-1)/s_w\rfloor\bigr)$
\For{$t=1,\ldots,N$}
    \State $B\gets\Call{PackGroups}{K,g}$ \Comment{Only for grouped convolutions}
    \State $A\gets\Call{Conv2D}{B,K;\,s,p,g}$
    \State $C\gets\Call{ConvTranspose2D}{A,K;\,s,p,g,0}$
        \Comment{Performs Convolution and cropping}
    \State $C\gets\Call{UnpackGroups}{C,g}$ \Comment{Only for grouped convolutions}
    \State $K\gets\alpha\bigl((1+\beta)K-\beta C\bigr)$
\EndFor
\State \Return $K/\alpha$
\end{algorithmic}
\end{algorithm}
We initialize the kernel by spectral rescaling using three Gram
iterations~\citep{delattre2023efficient, delattre2024spectral}. Group packing rearranges
the kernel so that grouped convolutions can process its output
filters as a batch; unpacking restores the original layout.
The convolution--transposed-convolution pair evaluates the cubic
Newton--Schulz correction on the original kernel support, with
cropping implemented implicitly by the padding choices.
We use a constant $\beta=0.5$: increasing $\beta$ can destabilize
the iteration, whereas decreasing it slows convergence.
Although the unprojected polynomial preserves the operator-norm
unit ball for $0<\beta\leq 0.5$, support cropping need not do so.
Damping by $\alpha=0.91$ provides a stability margin intended to
keep the iterates inside this ball. The final division by $\alpha$
removes the last damping factor.
The block-convolution derivation concerns ungrouped, unit-stride
convolutions; the grouped and strided extensions were checked
by unit tests.

\paragraph{Efficiency.}
Cropping preserves the original kernel support, giving each
iteration a fixed computational and memory cost for fixed kernel
dimensions, stride, and group count. The Conv-NS loop operates
directly on real-valued kernels and requires no FFT. It therefore
avoids complex-valued frequency-domain arithmetic and transforms
on the activation-sized grid, whose spatial dimensions are
typically much larger than the kernel support.

\section{Convergence analysis}
\label{sec:appendix-convergence}
\label{subsec:conv-ns-convergence}
\label{sec:appendix-newton-schulz}
\label{sec:appendix-gradient}

The convolutional Newton--Schulz update should be interpreted in two stages.
Without the support projection, it is exactly the classical Bj\"orck / Newton--Schulz iteration written in kernel coordinates.
Once the projection onto fixed spatial support is inserted, the method is no longer exact Newton--Schulz on the full convolution operator; it becomes a projected gradient method for a support-constrained orthogonality objective.
This objective coincides with the $\mathcal{L}_{\mathrm{orth}}$ surrogate of \citet{achour2022existence} in the stride-aware setting.

\paragraph{Exact block-convolution form of Newton--Schulz when $S=1$.}
Assume circular padding and unit stride.
Let
\[
\widehat K^{(t+1)}
=
(1+\beta)\widehat K^{(t)}
-
\beta\,\widehat K^{(t)} \circledast (\widehat K^{(t)})^\top \circledast \widehat K^{(t)},
\]
where the support of $\widehat K^{(t)}$ is allowed to grow with $t$.
Define
\[
X^{(t)} := \mathcal T_{\mathrm{circ}}(\widehat K^{(t)}).
\]

\begin{proposition}[Unprojected block-convolution Newton--Schulz]
For every $t$,
\[
X^{(t+1)}
=
(1+\beta)X^{(t)} - \beta X^{(t)}(X^{(t)})^\top X^{(t)}.
\]
Hence the unprojected kernel iteration is exactly the classical Bj\"orck / Newton--Schulz iteration on the induced convolution operator.
In particular, if $\|X^{(0)}\|_2 \le 1$ and $0<\beta\le \frac12$, then $X^{(t)}$ converges to the polar factor  $\operatorname{PF}(X^{(0)})$.
\end{proposition}

\begin{proof}
Using the block-convolution identities,
\[
\mathcal T_{\mathrm{circ}}(A\circledast B)=\mathcal T_{\mathrm{circ}}(A)\mathcal T_{\mathrm{circ}}(B),
\qquad
\mathcal T_{\mathrm{circ}}(K^\top)=\mathcal T_{\mathrm{circ}}(K)^\top,
\]
we obtain
\[
\mathcal T_{\mathrm{circ}}\!\bigl(
\widehat K^{(t)} \circledast (\widehat K^{(t)})^\top \circledast \widehat K^{(t)}
\bigr)
=
X^{(t)}(X^{(t)})^\top X^{(t)}.
\]
Applying $\mathcal T_{\mathrm{circ}}$ to the kernel recursion yields the claimed matrix iteration.
The convergence statement is then the standard Bj\"orck / Newton--Schulz result.
\end{proof}

\paragraph{Implicit Newton--Schulz as gradient descent on the orthogonality objective.}
Let $\mathcal M_{\mathrm{kernel}}$ denote the linear space of kernels with fixed spatial support
$k_h\times k_w$, and let $\Pi_{\mathrm{supp}}$ be the orthogonal projection onto that space
(center crop).
In the row-orthogonal case, define
\[
\Phi_{\mathrm{bc}}(K)
:=
\frac14 \|K\circledast K^\top - I\|_F^2,
\qquad K\in \mathcal M_{\mathrm{kernel}}.
\]
The column-orthogonal case is identical after replacing $K\circledast K^\top-I$ by
$K^\top\circledast K-I$.

\begin{proposition}[Gradient formula]
\label{prop:gradient-formula}
For every $K\in \mathcal M_{\mathrm{kernel}}$,
\[
\nabla \Phi_{\mathrm{bc}}(K)
=
\Pi_{\mathrm{supp}}\!\bigl((K\circledast K^\top-I)\circledast K\bigr)
=
\Pi_{\mathrm{supp}}\!\bigl(K\circledast K^\top\circledast K - K\bigr).
\]
Consequently,
\[
\Pi_{\mathrm{supp}}\!\Bigl((1+\beta)K-\beta K\circledast K^\top\circledast K\Bigr)
=
K-\beta \nabla \Phi_{\mathrm{bc}}(K).
\]
\end{proposition}

\begin{proof}
Let $H\in \mathcal M_{\mathrm{kernel}}$.
Since
\[
D(K\circledast K^\top)[H]
=
H\circledast K^\top + K\circledast H^\top,
\]
we have
\[
D\Phi_{\mathrm{bc}}(K)[H]
=
\frac12
\left\langle
K\circledast K^\top - I,\;
H\circledast K^\top + K\circledast H^\top
\right\rangle.
\]
Now use that the adjoint of convolution by $K^\top$ is convolution by $K$ followed by
the center crop onto the support of $H$:
\[
\langle H\circledast K^\top, Z\rangle
=
\langle H,\Pi_{\mathrm{supp}}(Z\circledast K)\rangle.
\]
Since $K\circledast K^\top-I$ is self-transpose, the two terms coincide, hence
\[
D\Phi_{\mathrm{bc}}(K)[H]
=
\left\langle
\Pi_{\mathrm{supp}}\!\bigl((K\circledast K^\top-I)\circledast K\bigr),
H
\right\rangle.
\]
This identifies the gradient.
The final identity follows because $\Pi_{\mathrm{supp}}(K)=K$.
\end{proof}

\begin{remark}[Why $\mathcal{L}_{\mathrm{orth}}$ is the right strided objective]
For stride $S>1$, exact Newton--Schulz on the actual layer operator must be written for
\[
A_S(K):=S_S\mathcal T(K),
\]
and its cubic term is
\[
A_S(K)A_S(K)^\top A_S(K)
=
S_S\mathcal T(K)\mathcal T(K)^\top S_S^\top S_S\mathcal T(K),
\]
which is not the same as $S_S\mathcal T(K\circledast K^\top\circledast K)$.
Therefore, in the strided case, the correct fixed-support objective is not the plain
unstrided block-convolution residual but the stride-aware regularizer $\mathcal{L}_{\mathrm{orth}}$
of \citep{achour2022existence}, whose gradient is implemented in practice by a standard
convolution followed by a transposed convolution.
The formula above is recovered exactly when $S=1$.
\end{remark}

\paragraph{Projected descent on the support-constrained $\mathcal{L}_{\mathrm{orth}}$ objective.}
In the row-orthogonal case define
\[
\Phi_S(K)
:=
\frac14\,\mathcal{L}_{\mathrm{orth}}(K)
=
\frac14\left\|
\mathrm{conv}(K,K,\mathrm{padding}=P,\mathrm{stride}=S)-I_{r0}
\right\|_F^2,
\]
with $P=\lfloor (k-1)/S\rfloor S$.
The practical update can then be written abstractly as
\[
K^{(t+1)}
=
K^{(t)}-\beta_t \nabla \Phi_S(K^{(t)}),
\]
or equivalently, in implementation form,
\[
K^{(t+1)}
=
\Pi_{\mathrm{supp}}\!\Bigl(K^{(t)}-\beta_t G(K^{(t)})\Bigr),
\]
where $G(K^{(t)})$ is obtained by one direct convolution and one transposed convolution.
When $S=1$, this reduces to the block-convolution formula above.

\begin{theorem}[Convergence to a stationary point]
Assume that the iterates remain in a bounded set $\mathcal D\subset \mathcal M_{\mathrm{kernel}}$.
Since $\Phi_S$ is a polynomial on the finite-dimensional space $\mathcal M_{\mathrm{kernel}}$,
its gradient is Lipschitz on $\mathcal D$; let $L$ be such a Lipschitz constant.
If the step sizes satisfy
\[
0<\underline\beta \le \beta_t \le \overline\beta < \frac{2}{L},
\]
then
\[
\Phi_S(K^{(t+1)})
\le
\Phi_S(K^{(t)})
-
c\|\nabla \Phi_S(K^{(t)})\|_F^2,
\qquad
c:=\underline\beta\Bigl(1-\frac{L\overline\beta}{2}\Bigr)>0.
\]
Consequently,
\[
\sum_{t=0}^{\infty}\|\nabla \Phi_S(K^{(t)})\|_F^2 < \infty,
\qquad
\nabla \Phi_S(K^{(t)})\to 0,
\]
and every accumulation point of $(K^{(t)})_{t\ge 0}$ is a critical point of the support-constrained problem
\[
\min_{K\in \mathcal M_{\mathrm{kernel}}} \Phi_S(K).
\]
Moreover, because $\Phi_S$ is a polynomial (hence a Kurdyka--\L{}ojasiewicz function),
the whole sequence $(K^{(t)})_{t\ge 0}$ converges to a single critical point
$K^\star$ of $\Phi_S$.
\end{theorem}

\begin{proof}
By the descent lemma for $L$-smooth functions,
\[
\Phi_S(K-\beta\nabla \Phi_S(K))
\le
\Phi_S(K)-\beta\|\nabla \Phi_S(K)\|_F^2
+\frac{L\beta^2}{2}\|\nabla \Phi_S(K)\|_F^2.
\]
Applying this with $K=K^{(t)}$ and $\beta=\beta_t$ gives
\[
\Phi_S(K^{(t+1)})
\le
\Phi_S(K^{(t)})
-
\beta_t\Bigl(1-\frac{L\beta_t}{2}\Bigr)\|\nabla \Phi_S(K^{(t)})\|_F^2,
\]
hence the announced inequality.
Summing over $t$ proves that $\sum_t \|\nabla \Phi_S(K^{(t)})\|_F^2<\infty$, and therefore
$\nabla \Phi_S(K^{(t)})\to 0$.
Criticality of accumulation points follows by continuity of $\nabla \Phi_S$.
Finally, the KL property of polynomial functions upgrades subsequential convergence to convergence
of the whole bounded sequence.
\end{proof}

\begin{corollary}[Operator-level consequence]
Under circular boundary conditions and $SN\ge 2k-1$, the results summarized in the
\textsc{Lorth} analysis imply that $\mathcal{L}_{\mathrm{orth}}$ controls both the Frobenius and spectral
orthogonality errors of the induced convolution operator.
Therefore, if $\Phi_S(K^{(t)})\to 0$, then the operator-level orthogonality residuals vanish.
More generally, if the terminal value $\Phi_S(K^{(T)})$ is small, then the spectral residual
is small, and the induced convolution is approximately isometric.
Thus the projected convolutional Newton--Schulz iteration should be viewed as a
support-constrained orthogonality solver: when an orthogonal kernel exists, it can converge
to such a solution; otherwise it converges to the best stationary approximation inside the
fixed-support kernel class.
\end{corollary}

\paragraph{Initial conditions.}
To guarantee the convergence of the Newton--Schulz iteration, the initial operator norm of the convolution must be strictly bounded, satisfying $\|\mathcal{T}(K)\|_2 \le 1.0$.
Although a simple Frobenius normalization of the kernel tensor is sufficient to achieve this, employing a tighter spectral estimate accelerates convergence.
The accurate evaluation of the Lipschitz constant for convolutional layers is an active area of research in its own right: initial approaches leveraged the Fast Fourier Transform (FFT) to compute a lower bound for layers with circular padding~\citep{sedghi_singular_2019}.
This foundation was subsequently refined by methods offering improved computational efficiency~\citep{senderovich_towards_2022}, support for diverse padding strategies~\citep{grishina_tight_2024}, and the ability to extract broader regions of the operator spectrum~\citep{boroojeny_spectrum_2024}.
In our implementation, we rely on the techniques established by \citet{delattre2023efficient} and \citet{delattre2024spectral}, which allow us to efficiently compute a certifiable upper bound on the spectral norm across various padding modes.

Finally, while the target implies that $\|\mathcal{T}(K)\|_2 \le 1.0$, the projection step, coupled with numerical imprecision, can yield an iterate with $\|\mathcal{T}(K^{(t)})\|_2 \ge 1.0$.
To avoid this, we add a damping factor; when $\beta = 0.5$, a damping factor of $0.9$ yields good results.

\paragraph{Convergence guarantees:} 
Without the support projection, the block-convolution Newton-Schulz iteration is \emph{exactly} the classical Bj\"orck / Newton--Schulz iteration on the induced operator $\mathcal{T}
(\Delta_K)$, and therefore converges to its polar factor whenever $\|\mathcal{T}(\Delta_K)\|_2 \le 1$ and $\beta \le 0.5$.
Once the projection $\Pi_{\mathrm{supp}}$ is inserted, the method is no longer exact Newton--Schulz on the full operator; instead, it becomes projected gradient descent on a support-constrained orthogonality objective $\Phi_S=\|
\mathcal{T}(\Delta)
\mathcal{T}(\Delta)^T-I\|^2$ that coincides with the $\mathcal{L}_{\mathrm{orth}}$ surrogate of \citet{achour2022existence} in the strided setting (\cref{prop:gradient-formula}).
Because $\Phi_S$ is a polynomial on the finite-dimensional kernel space, its gradient is Lipschitz and it satisfies the Kurdyka--\L{}ojasiewicz property: the iterates are guaranteed to converge to a single critical point of $\Phi_S$, with $\sum_t \|\nabla \Phi_S(\Delta_K^{(t)})\|_F^2 < \infty$ \citep{bolte2010characterizations}.
In particular, when an orthogonal kernel exists within the fixed-support class, the method can converge to it; otherwise it converges to the best stationary approximation.
The initial condition $\|\mathcal{T}(\Delta_K)\|_2 \le 1$ is enforced via a spectral-norm upper bound computed efficiently following \citet{delattre2023efficient,delattre2024spectral}, and a damping factor handles residual numerical drift across iterations.

\section{Correction of the grafting factor:} 
\label{app:grabbing}
Typical implementations of Muon multiply the orthogonal update by a scalar factor based on the matrix dimensions. While there is no strict consensus, these grafting factors allow balancing learning between layers and help with learning rate transferability. Following the logic of this paper, these factors should be computed according to the dimensions of the induced operator $\mathcal{T}(\Delta_K)$ rather than the dimensions of the reshaped kernel $K$. \Cref{tab:scaling-factors} presents different scaling conventions and their adaptation. Note that \citet{bernstein_modular_2025} uses a different logic than the other two by proposing distinct formulas for dense and convolutional layers, where the dense formula does not rely on the kernel sizes $k_w$ and $k_h$.

\begin{table}[t]
    \centering
    \small
    \begin{tabular}{llll}
        \hline
        & \citet{jordan_muon_2024}
        & \citet{liu_muon_2025}
        & \citet{bernstein_modular_2025} \\
       \hline

        Dense matrix 
        & $\sqrt{\max\left(1,\frac{n}{m}\right)}$
        & $0.2\sqrt{\max(n,m)}$
        & $\sqrt{\frac{n}{m}}$ \\
        ($n \times m$) & & & \\

        Reshaped 
        & $\sqrt{\max\left(1,\frac{c_{\mathrm{out}}}
        {c_{\mathrm{in}}k_wk_h}\right)}$
        & $0.2\sqrt{\max\left(c_{\mathrm{out}},
        c_{\mathrm{in}}k_wk_h\right)}$
        & $\frac{1}{k_wk_h}
        \sqrt{\frac{c_{\mathrm{out}}}{c_{\mathrm{in}}}}$ \\
        kernel & & &\\
        Adaptation
        & $\sqrt{\max\left(1,
        \frac{c_{\mathrm{out}}}{c_{\mathrm{in}}s^2}\right)}$
        & $0.2\sqrt{\max\left(
        c_{\mathrm{in}}i_wi_h,
        c_{\mathrm{out}}\frac{i_w}{s}\frac{i_h}{s}
        \right)}$
        & $\sqrt{\frac{c_{\mathrm{out}}}
        {c_{\mathrm{in}}s^2}}$ \\
        for $\mathcal{T}(\Delta_K)$ & & &\\
        \hline
    \end{tabular}

    \caption{Grafting factors for Muon updates. The adaptation for
    $\mathcal{T}(\Delta_K)$ applies the dense-matrix formula using the
    dimensions of the induced convolution operator,
    $n=c_{\mathrm{out}}\frac{i_h}{s}\frac{i_w}{s}$ and
    $m=c_{\mathrm{in}}i_hi_w$.}
    \label{tab:scaling-factors}
\end{table}
\section{CIFAR-10 and ImageNet classification details}
\label{sec:appendix-classification-details}

\paragraph{CIFAR-10 AirBench94.}
The CIFAR-10 experiment uses the AirBench94 benchmark~\citep{jordan_airbench_2024}, a speed-optimized 8-epoch training recipe designed to reach 94\% accuracy in only a few seconds on a single GPU.
The network is the compact AirBench CNN with three convolutional groups (64$\to$256$\to$256 channels) built from $3{\times}3$ stride-1 convolutions.
Runs use batch size 2000, 8 training epochs, 3 whitening epochs, horizontal-flip augmentation, and translation augmentation with padding 2.
Conv-NS and reshape-based Muon use the same architecture, training schedule, augmentation pipeline, batch size, and non-convolutional optimizer groups; only the convolutional orthogonalization routine inside Muon differs.
The SGD baseline replaces Muon entirely.

For reshape-based Muon and SGD, we use the original AirBench author's benchmark-specific optimized settings.
Because no reference Conv-NS setting existed for this benchmark, we tune Conv-NS with a 1000-configuration search budget.
The selected Conv-NS configuration uses Muon learning rate $0.52$, Muon momentum $0.6$, zero Muon weight decay, bias learning rate $0.053$, head learning rate $0.67$, Adam weight-decay scale $6{\times}10^{-6}$, and SGD momentum $0.85$ for the remaining parameter groups.
For the convolutional orthogonalization itself, we use $\beta_{\mathrm{init}}=0.25$, $\beta_{\mathrm{end}}=0.5$, 25 Conv-NS iterations, damping $0.99$, $\gamma=0.01$, and $\epsilon=0.01$.

Conv-NS, reshape-based Muon, and SGD are each reported over 25 independent runs.
We report both ordinary validation accuracy and test-time-augmented validation accuracy.
On a single RTX4090, the non-optimized PyTorch implementation of Conv-NS runs in $4.45$ seconds per run, compared with $3.68$ seconds for reshape-based. Runtime with torch compilation enabled are reported in main paper.

\paragraph{ImageNet-1k ResNet-50.}
The ImageNet experiment uses ImageNet-1k~\citep{deng2009imagenet} with a ResNet-50 backbone and a \texttt{timm}-style training recipe shortened to 30 epochs.
Training uses batch size 512 across 4 H100 GPUs (so global batch size is 2048), a linear learning-rate schedule with no warmup, TA-Wide augmentation, mixup ($\alpha{=}0.1$), and CutMix ($\alpha{=}1.0$).
We compare the \texttt{flatten} strategy (reshape-based Muon) against the \texttt{bjorck} strategy (Conv-NS).

Both strategies receive the same grid-search budget over learning rate $\in \{0.01, 0.02, 0.05, 0.1\}$, momentum $\in \{0.8, 0.9, 0.95\}$, weight decay $\in \{2{\times}10^{-5}, 10^{-4}\}$, and cautious weight decay $\in \{\text{true}, \text{false}\}$.
This gives 48 configurations per strategy and 96 runs in total.
Both strategies reach their best validation top-1 accuracy at lr$=0.02$, momentum$=0.95$, weight decay$=2{\times}10^{-5}$, with cautious weight decay enabled.

\section{Per-stage Pareto analysis on real training gradients}
\label{sec:appendix-per-stage-pareto}

\Cref{subsec:why-reshape-works} reports the orthogonality--fidelity Pareto front averaged over three training stages.
\Cref{fig:grad-pareto-per-stage} shows the same analysis separately for gradient snapshots captured at the beginning, middle, and end of CIFAR-10 training.
The key observation---that reshape-based Muon (RKO) lies on or very near the Pareto front for real gradients---holds consistently across all three stages, confirming that this is not an artifact of a particular training phase.

\begin{figure}[h]
    \centering
    \includegraphics[width=0.96\textwidth]{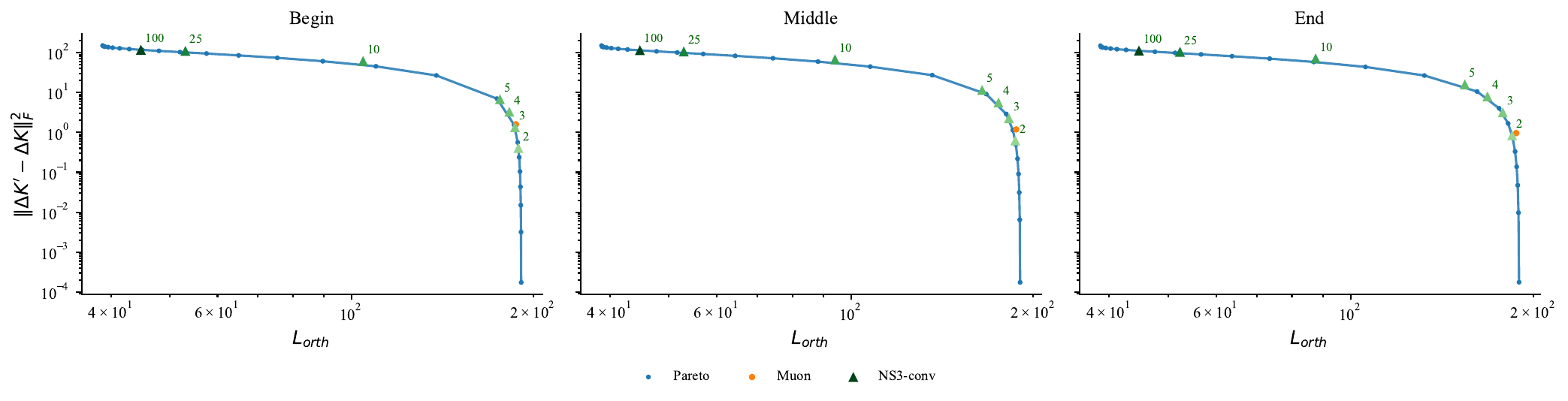}
    \caption{\textbf{Per-stage Pareto trade-off on real training gradients.} Each panel shows the Pareto front computed on gradient snapshots from a single training stage (begin, middle, end). Reshape-based Muon (RKO) lies on the Pareto front at all stages, and Conv-NS traces the front as the iteration count increases.}
    \label{fig:grad-pareto-per-stage}
\end{figure}

\section{Hierarchical optimization of orthogonality objectives}
\label{sec:appendix-hierarchical}

This section provides the full numerical results comparing the orthogonality residuals of different operators, as discussed in the main text. \Cref{tab:rko-path-synthetic} presents the results for synthetic, randomly initialized kernels, while \Cref{tab:rko-path-real} shows the results for real gradient snapshots captured during training.

On synthetic kernels, \cref{tab:rko-path-synthetic} illustrates an asymmetric relationship between the two orthogonality objectives. RKO nearly eliminates the flattened-kernel residual ($\mathcal{L}_{\mathrm{flat}}=0.02$), but leaves substantial convolutional error ($\mathcal{L}_{\mathrm{orth}}=14.40$). Rescaling RKO to unit operator norm reduces this error to $9.13$ while sacrificing flattened orthogonality, showing that spectral normalization alone does not resolve the mismatch. In contrast, ConvPF achieves low residuals on both objectives ($0.10$ and $0.02$, respectively), while ten Conv-NS iterations yield intermediate values ($3.38$ and $1.11$). These results show that, in this setting, convolutional orthogonalization can also satisfy the flattened-kernel objective, whereas optimizing the latter alone is insufficient to ensure convolutional orthogonality.

\begin{table}[h]
    \centering
    \small
    \begin{tabular}{lrrrr}
\hline
Node & Trials & Lipschitz & $L_{\mathrm{orth}}$ & $L_{\mathrm{flat}}$ \\
\hline
K & 10 & $2.31 \pm 0.07$ & $36.57 \pm 1.53$ & $7.11 \pm 0.69$ \\
$RKO(\Delta K)$ & 10 & $1.83 \pm 0.03$ & $14.40 \pm 0.40$ & $0.02 \pm 0.00$ \\
$RKO(\Delta K) / \sigma(RKO(\Delta K))$ & 10 & $1.00 \pm 0.00$ & $9.13 \pm 0.14$ & $7.85 \pm 0.20$ \\
$ConvPF(\Delta K)$ & 10 & $1.07 \pm 0.02$ & $0.10 \pm 0.03$ & $0.02 \pm 0.00$ \\
$ConvNS_{10}(\Delta K)$ & 10 & $1.14 \pm 0.01$ & $3.38 \pm 0.09$ & $1.11 \pm 0.04$ \\
\hline
\end{tabular}

    \vspace{0.5em}
    \caption{\textbf{Orthogonality hierarchy on synthetic data.} Highlighting the difference between the convolutional polar factor and the reshape-based orthogonalization (RKO). By measuring both the operator-level orthogonality residual $\mathcal{L}_{\mathrm{orth}}$ (Polar factor objective) and the flattened-kernel orthogonality residual $\mathcal{L}_{\mathrm{flat}}$ (RKO objective) we can observe a hierarchy between the two: the convolutional polar factor achieves both operator-level and flattened orthogonality, while RKO only satisfies the latter.}
    \label{tab:rko-path-synthetic}
\end{table}

On real training gradients, \cref{tab:rko-path-real} confirms the mismatch between flattened and convolutional orthogonality across all three training stages: RKO maintains a low flattened-kernel residual, yet its operator norm remains close to $3$, compared with $1.83$ in the synthetic setting. ConvPF consistently achieves near-zero convolutional error, while ten Conv-NS iterations provide an intermediate level of orthogonalization. However, the relationship between the two objectives is less clear-cut than on synthetic kernels: ConvPF yields $\mathcal{L}_{\mathrm{flat}}\approx29$, substantially above RKO's $0.22$--$0.51$. Thus, although the convolutional-orthogonality ordering persists throughout training, these measurements do not establish that ConvPF also minimizes the flattened-kernel residual on real gradients.

\begin{table}[t]
    \centering
    \small
    \begin{tabular}{llrrrr}
\hline
Stage & Node & Trials & Lipschitz & $L_{\mathrm{orth}}$ & $L_{\mathrm{flat}}$ \\
\hline
Begin & $\Delta K$ & 6 & $137.34 \pm 47.94$ & 5.2e+07 & 1.1e+07 \\
Begin & $\Delta K / \sigma(\Delta K)$ & 6 & $1.00 \pm 0.00$ & $152.45 \pm 119.32$ & $191.04 \pm 104.36$ \\
Begin & $RKO(\Delta K)$ & 6 & $2.99 \pm 0.04$ & $461.04 \pm 304.04$ & $0.51 \pm 0.78$ \\
Begin & $RKO(\Delta K) / \sigma(RKO(\Delta K))$ & 6 & $1.00 \pm 0.00$ & $118.74 \pm 103.28$ & $151.18 \pm 81.79$ \\
Begin & $ConvPF(\Delta K)$ & 6 & $1.08 \pm 0.07$ & $0.09 \pm 0.08$ & $29.26 \pm 62.44$ \\
Begin & $ConvNS_{10}(\Delta K)$ & 6 & $1.28 \pm 0.09$ & $108.20 \pm 101.75$ & $137.14 \pm 87.38$ \\
Middle & $\Delta K$ & 6 & $119.15 \pm 47.63$ & 4.7e+07 & 1.2e+07 \\
Middle & $\Delta K / \sigma(\Delta K)$ & 6 & $1.00 \pm 0.00$ & $150.68 \pm 117.47$ & $189.18 \pm 102.69$ \\
Middle & $RKO(\Delta K)$ & 6 & $2.94 \pm 0.12$ & $471.96 \pm 287.20$ & $0.22 \pm 0.11$ \\
Middle & $RKO(\Delta K) / \sigma(RKO(\Delta K))$ & 6 & $1.00 \pm 0.00$ & $117.73 \pm 101.64$ & $149.60 \pm 80.70$ \\
Middle & $ConvPF(\Delta K)$ & 6 & $1.05 \pm 0.07$ & $0.04 \pm 0.07$ & $29.25 \pm 62.47$ \\
Middle & $ConvNS_{10}(\Delta K)$ & 6 & $1.26 \pm 0.10$ & $71.22 \pm 61.30$ & $87.57 \pm 58.03$ \\
End & $\Delta K$ & 6 & $126.52 \pm 50.66$ & 6.0e+07 & 1.5e+07 \\
End & $\Delta K / \sigma(\Delta K)$ & 6 & $1.00 \pm 0.00$ & $150.98 \pm 117.89$ & $189.51 \pm 103.03$ \\
End & $RKO(\Delta K)$ & 6 & $2.95 \pm 0.11$ & $457.84 \pm 281.82$ & $0.22 \pm 0.12$ \\
End & $RKO(\Delta K) / \sigma(RKO(\Delta K))$ & 6 & $1.00 \pm 0.00$ & $117.51 \pm 101.34$ & $149.55 \pm 80.27$ \\
End & $ConvPF(\Delta K)$ & 6 & $1.04 \pm 0.05$ & $0.08 \pm 0.07$ & $29.25 \pm 62.45$ \\
End & $ConvNS_{10}(\Delta K)$ & 6 & $1.26 \pm 0.10$ & $72.59 \pm 64.10$ & $89.26 \pm 59.25$ \\
\hline
\end{tabular}

    \vspace{0.5em}
    \caption{\textbf{Orthogonality hierarchy on real gradients.} The same comparison as in \cref{tab:rko-path-synthetic}, evaluated on real gradient snapshots from a CIFAR-10 training run across different stages (begin, middle, end). The hierarchy between the objectives remains consistent with the synthetic setting.}
    \label{tab:rko-path-real}
\end{table}

\section{Effect of iteration count in the 8-epoch setting}
\label{sec:appendix-8-epoch}

In \Cref{fig:cifar-acc-vs-niter} the 8-epoch CIFAR-10 AirBench setting is reused to evaluate the effect of different iteration counts. For a given iteration count, the learning rate is swept along the momentum. Five runs are conducted per learning rate, and the momentum with the best average  (test time augmentation) validation accuracy across learning rates is selected for each iteration count; shaded bands show $\pm 1$ standard deviation over these five runs.
While low iteration budget exhibits a similar behavior observed by \citep{shulgin_beyond_2026}: a larger number of iterations improves performance with a shift in optimal lr and momentum, which can be interpreted as the result of a faster training.

\begin{figure}[ht]
    \centering
    \includegraphics[width=\columnwidth]{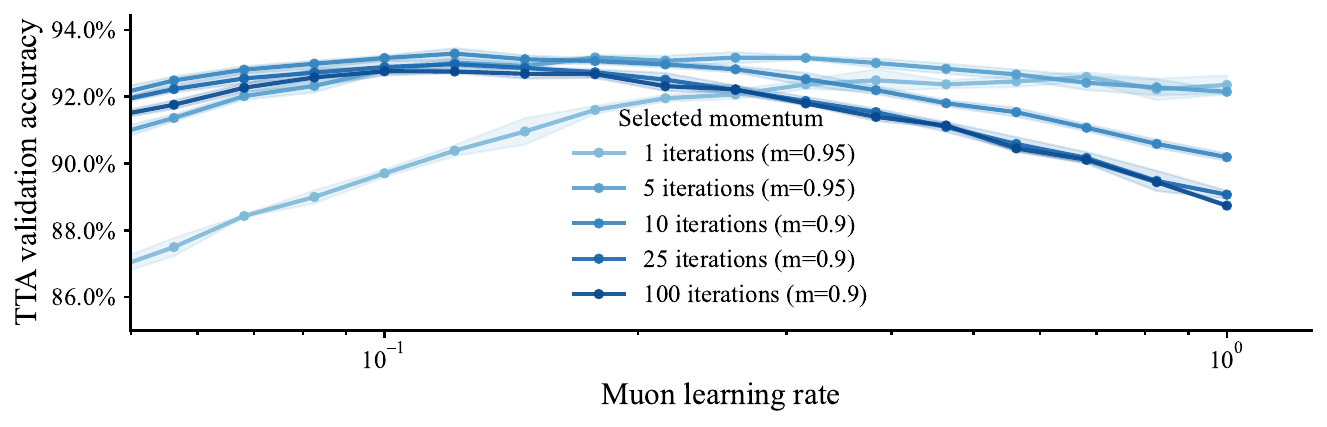}
    \caption{\textbf{Effect of Conv-NS iteration count on CIFAR-10 AirBench.} TTA validation accuracy over the Muon learning-rate sweep for different numbers of Newton--Schulz iterations. For each iteration count, momentum is selected by the best average TTA validation accuracy across learning rates; shaded bands show $\pm 1$ standard deviation over five runs.}
    \label{fig:cifar-acc-vs-niter}
\end{figure}

\section{One-epoch selection criteria}
\label{sec:appendix-one-epoch-criteria}

\paragraph{Robustness to the selection criterion.}
To verify that the non-monotonic effect of Conv-NS iteration count is not an artifact of the hyperparameter-selection metric, we reselect the best configurations using either TTA (test time augmentation) validation accuracy or training accuracy (\Cref{tab:1-epoch-selection-criteria}). In both cases, performance improves up to 10 iterations and degrades at 25 and 100 iterations, despite substantial shifts in the selected learning rate. Moreover, the same trend under training-accuracy selection shows that the degradation is already present at the optimization level rather than arising solely from generalization. These results indicate that the Conv-NS iteration count is not merely an approximation-accuracy parameter: pushing the update closer to exact convolutional orthogonality can itself hinder training.

\begin{table}[h]
    \centering
    \small
    \setlength{\tabcolsep}{5pt}
    \begin{tabular}{lccccc}
        \toprule
        \multicolumn{6}{l}{\textbf{Top by TTA validation accuracy}} \\
        Iterations (Conv-NS) & 1 & 5 & 10 & 25 & 100 \\
        \midrule
        TTA Val. Acc. (\%) & 85.36 & 86.66 & \textbf{87.75} & 86.61 & 85.90 \\
        Val. Acc. (\%) & 84.67 & 85.91 & 86.84 & 85.58 & 84.97 \\
        Train Acc. (\%) & 84.44 & 86.20 & 85.44 & 85.44 & 85.76 \\
        Muon LR & 0.433 & 0.432 & 0.381 & 0.181 & 0.143 \\
        \bottomrule
    \end{tabular}

    \vspace{0.75em}
    \begin{tabular}{lccccc}
        \toprule
        \multicolumn{6}{l}{\textbf{Top by training accuracy}} \\
        Iterations (Conv-NS) & 1 & 5 & 10 & 25 & 100 \\
        \midrule
        TTA Val. Acc. (\%) & 84.93 & 86.46 & 87.55 & 86.05 & 85.13 \\
        Val. Acc. (\%) & 84.22 & 85.71 & 86.69 & 85.29 & 83.94 \\
        Train Acc. (\%) & 86.32 & 87.32 & \textbf{88.12} & 87.00 & 86.32 \\
        Muon LR & 0.579 & 0.194 & 0.441 & 0.158 & 0.094 \\
        \bottomrule
    \end{tabular}
    \vspace{0.5em}
    \caption{\textbf{Alternate one-epoch CIFAR-10 selection criteria.} These blocks report the single top configuration per sweep under validation-oriented and train-accuracy-oriented selection criteria; each configuration line is already averaged over repeated runs with the same hyperparameters.}
    \label{tab:1-epoch-selection-criteria}
\end{table}

\section{Comparison with One Pass ImageNet}
\label{app:one-pass-imagenet}

One Pass ImageNet (OPIN; \citet{hu2021one}) studies ResNet-50 training from a
randomly ordered ImageNet stream in which each training example arrives once.
Its one-pass constraint permits retaining and revisiting examples in a limited
replay buffer. The proposed error-prioritized replay method preferentially
samples examples with larger predictive errors and applies importance weights
to correct the sampling bias. With $k$ replay minibatches per incoming minibatch,
the effective training budget is $E_{\mathrm{eff}}=k+1$ epochs. OPIN Table~2
reports results at 2, 4, 6, and 9 effective epochs, using buffers containing
1\%, 5\%, or 10\% of the dataset, alongside full-storage multi-epoch baselines.

We compare these results with our one-epoch ImageNet experiments, which use
no replay and therefore have $E_{\mathrm{eff}}=1$. To distinguish data-processing
cost from optimization steps, we report both budgets relative to a 90-epoch
run with batch size 128. For dataset size $N$ and global batch size $B$, our batch size of 512 gives approximately 1.11\% of the example-processing budget
and 0.278\% of the update budget. These quantities describe a single training
run and exclude hyperparameter search. The computation proxy does not account
for optimizer overhead and should not be interpreted as measured FLOPs or
wall-clock time.

\begin{table*}[b]
  \centering
  \scriptsize
  \setlength{\tabcolsep}{3pt}
  \resizebox{\textwidth}{!}{%
  \begin{tabular}{lrrrrrrrrrr}
    \toprule
    & & & & & & \multicolumn{5}{c}{Validation top-1 accuracy (\%)} \\
    \cmidrule(lr){7-11}
    Method & \shortstack{Effective\\epochs}
    & \shortstack{Computation\\proxy}
    & \shortstack{Relative\\updates}
    & \shortstack{Batch\\size}
    & \shortstack{Updates\\(approx.)}
    & \shortstack{No\\replay}
    & \shortstack{Replay\\1\%}
    & \shortstack{Replay\\5\%}
    & \shortstack{Replay\\10\%}
    & \shortstack{Multi-epoch\\100\%} \\
    \midrule
    OPIN & 2 & $2/90$ & $2/90$ & 128 & 20,018 & -- & 44.7 & 45.1 & 45.7 & 46.1 \\
    OPIN & 4 & $4/90$ & $4/90$ & 128 & 40,036 & -- & 55.5 & 57.1 & 57.2 & 59.0 \\
    OPIN & 6 & $6/90$ & $6/90$ & 128 & 60,055 & -- & 58.9 & 61.3 & 62.2 & 64.1 \\
    OPIN & 9 & $9/90$ & $9/90$ & 128 & 90,082 & -- & 59.3 & 63.2 & 65.0 & 68.2 \\
    \midrule
    Muon (reshape), ours & 1 & $1/90$ & $1/360$ & 512 & 2,502 & \textbf{54.05} & -- & -- & -- & -- \\
    Muon (Conv-NS), ours & 1 & $1/90$ & $1/360$ & 512 & 2,502 & 52.95 & -- & -- & -- & -- \\
    SGD, our recipe & 1 & $1/90$ & $1/360$ & 512 & 2,502 & 30.47 & -- & -- & -- & -- \\
    \bottomrule
  \end{tabular}%
  }
  \caption{Comparison of OPIN Table~2 with the one-epoch ImageNet results
  in our Table~1. All experiments use ResNet-50. Replay column percentages
  indicate buffer capacity relative to the dataset; the multi-epoch baseline
  stores the full dataset. Our runs require no replay storage. Both budget
  ratios use a 90-epoch, batch-128 reference. Update counts are estimates
  using $N=1{,}281{,}167$, rounded to the nearest integer; exact counts depend
  on final-batch handling. The OPIN counts interpret each reported replay
  step as a separate minibatch update.}
  \label{tab:one-pass-imagenet}
\end{table*}

Both Muon variants exceed all OPIN results at two effective epochs,
including the full-storage baseline (46.1\%), while processing half as many
training examples and taking approximately eight times fewer optimizer
updates, without a replay buffer. They remain below the four-effective-epoch
replay results (55.5--57.2\%). Thus, the comparison highlights strong early
learning under a smaller data-processing and update budget. Because the
training recipes differ, it is a cross-paper comparison rather than a
controlled optimizer ablation; the resource ratios alone do not establish
an equivalent wall-clock speedup.


\section{Community baseline comparisons}
\label{sec:appendix-community-baselines}

\begin{figure*}[t]
    \centering
    \includegraphics[width=0.96\textwidth]{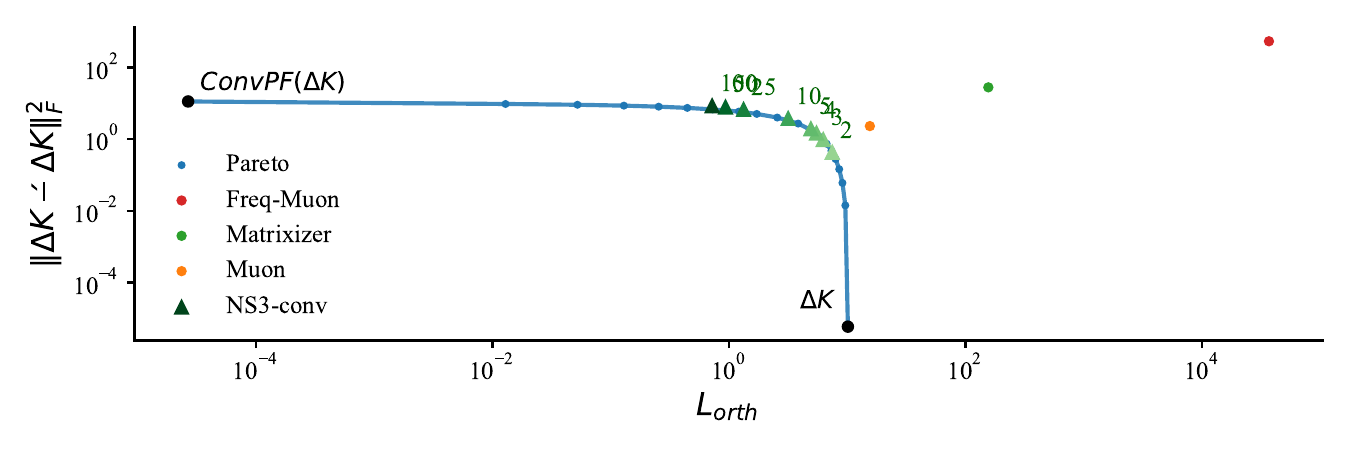}
    \caption{\textbf{Pareto trade-off on synthetic Gaussian kernels.} The grey curve is the empirical Pareto front computed via simulated annealing. Conv-NS iterates (blue dots, labelled by iteration count) trace the front from the rescaled SGD corner toward the convolutional polar factor as the number of iterations increases. Reshape-based Muon (RKO) lies far inside the suboptimal zone, reflecting a structural misalignment with the operator-level objective on random kernels. Matrixizer and Freq-Muon are far off-scale.}
    \label{fig:pareto-baselines}
\end{figure*}

We compare Conv-NS against two recent community-driven baselines -the only publicly available attempts at convolution-aware Muon to date -neither of which has been peer-reviewed or formally evaluated.
These comparisons are included for context, but remain appendix material because the main empirical claims center on Conv-NS, reshape-based Muon, and SGD.

\paragraph{Freq-Muon~\citep{kim2026frequency,cesista2026freqmuon}.}
Under circular boundary conditions, a convolution operator is block-circulant and is therefore diagonalized by the spatial DFT into $M^2$ independent channel-mixing matrices $\hat{K}[\omega] \in \mathbb{C}^{c_{\mathrm{out}} \times c_{\mathrm{in}}}$, one per frequency bin $\omega$.
Freq-Muon applies the standard Muon NS5 polar projection independently to each $\hat{K}[\omega]$ and recovers the spatial kernel via inverse FFT.
When the FFT grid matches the full input spatial dimensions ($M = i_h = i_w$), this procedure yields an exact orthogonalization of the circular convolution operator.
In practice, however, $M$ is set to a small fixed value (e.g.\ $M{=}8$) to keep costs tractable, breaking the exact operator correspondence and turning the method into a surrogate.
Additionally, the inverse FFT result is cropped to the original $k_h \times k_w$ support, discarding spatial taps beyond the kernel size.

\paragraph{Matrixizer~\citep{opooladz_matrixizer}.}
Instead of the standard $c_{\mathrm{out}} \times (c_{\mathrm{in}} k_h k_w)$ flattening used in reshape-based Muon, Matrixizer searches over all permutations of the gradient tensor's dimensions and all possible split points to find the 2D factorization that minimizes $m^{2}+n^{2}$, i.e.\ the total preconditioner cost (following the PSGD logic of~\citet{opooladz_matrixizer}).
The standard Muon NS5 iteration is then applied to this optimally shaped matrix.
While the resulting reshape is typically better balanced than the default flattening, the method remains fundamentally reshape-based: it does not account for stride, groups, dilation, or the structured embedding $\Delta K \mapsto \mathcal{T}(\Delta K)$ that governs how the convolution acts on data.

\subsection{Synthetic analysis of the orthogonality--fidelity trade-off}
\label{subsec:pareto-analysis}

\paragraph{Computing the empirical Pareto front.}
To evaluate how different orthogonalization strategies navigate the inherent trade-off between operator-level orthogonality ($\mathcal{L}_{\mathrm{orth}}$) and fidelity to the original kernel ($\mathcal{L}_{\mathrm{frob}} = \|\Delta K' - \Delta K\|_F^2$), we first construct a ground-truth Pareto front on synthetic data.
We initialize a random Gaussian convolution kernel with $c_{\mathrm{in}} = 16$, $c_{\mathrm{out}} = 16$, and $k_h=k_w=3$.
Because the scalarized objective $\lambda \mathcal{L}_{\mathrm{orth}} + (1 - \lambda) \mathcal{L}_{\mathrm{frob}}$ offers no convexity guarantees, we rely on a simulated annealing approach to explore the landscape.
We sample 50 scalar weightings $\lambda$ and optimize each configuration using Adam with simulated annealing noise injection over 1500 iterations.
The best result is kept across multiple trials to avoid local minima, yielding an empirical Pareto front that serves as a rigorous baseline.

\paragraph{Evaluating baselines on the synthetic landscape.}
Against this optimal trade-off curve, we overlay the operating points of Convolutional Newton--Schulz (Conv-NS) at iteration counts $\{2, 3, 4, 5, 10, 25, 50, 100\}$, reshape-based Muon (RKO), Matrixizer~\citep{opooladz_matrixizer}, and Freq-Muon~\citep{cesista2026freqmuon}.
The results are reported in \cref{fig:pareto-baselines}.

\paragraph{Conv-NS traces the Pareto front.}
As the iteration count grows from 2 to 100, the Conv-NS operating points slide along the simulated-annealing Pareto front, progressing from the low-modification region (high $\mathcal{L}_{\mathrm{orth}}$, low $\mathcal{L}_{\mathrm{frob}}$) toward the exact convolutional polar factor (low $\mathcal{L}_{\mathrm{orth}}$, high $\mathcal{L}_{\mathrm{frob}}$).
This confirms that the projected iteration efficiently navigates the support-constrained landscape and that the iteration count serves as a user-controllable dial for the orthogonality--fidelity trade-off.

\paragraph{Muon (RKO) is misaligned with the operator-level objective.}
Unlike its behavior on real training gradients (\cref{fig:grad-pareto-real}), reshape-based Muon sits far above the Pareto front on synthetic kernels. It modifies the kernel ($\mathcal{L}_{\mathrm{frob}} \approx 2.3$) without improving operator-level orthogonality ($\mathcal{L}_{\mathrm{orth}} \approx 15.6$, compared to ${\sim}10.2$ for the unmodified kernel).
Rather than a systematic degradation, this reflects a structural \emph{misalignment}: on unstructured random Gaussian data, the reshape surrogate pushes the kernel in a direction that does not coincide with the true operator-level objective.
A contributing factor is operator-norm inflation---orthogonalizing the flattened matrix can increase the spectral norm of the induced convolution---though the magnitude of this effect varies across initializations.
Importantly, as shown in \cref{subsec:why-reshape-works}, the structure of real training gradients corrects this misalignment, allowing RKO to fall surprisingly close to the Pareto front in practice.

\paragraph{Matrixizer and Freq-Muon are far from optimal.}
Matrixizer ($\mathcal{L}_{\mathrm{orth}} \approx 157$) and Freq-Muon ($\mathcal{L}_{\mathrm{orth}} \approx 3.7 \times 10^4$) lie orders of magnitude above the Pareto front, indicating that the approximations underlying these methods---Kronecker factorization and partial-spectrum FFT orthogonalization, respectively---introduce substantial artifacts on this objective.

\paragraph{Practical take-away.}
Even a modest iteration budget (e.g.\ 5 iterations) already places Conv-NS near the Pareto front at a useful orthogonality level ($\mathcal{L}_{\mathrm{orth}} \approx 4.9$), offering a practical compromise between compute cost and operator-level fidelity.
In contrast, no single-shot baseline achieves a comparable position on the front.

\subsection{CIFAR-10 community baseline results}
\label{subsec:cifar10-community-baselines}

All Muon-family methods: Conv-NS, reshape-based Muon, Matrixizer, and Freq-Muon, share the same architecture, training schedule, and augmentation pipeline; only the orthogonalization routine inside the optimizer differs.
Muon - Conv-NS and Muon - reshape are each reported over 25 independent runs (results slightly differs from \cref{tab:cifar10-imagenet-results} which report statistics over 10 runs); Matrixizer, Freq-Muon, and SGD are preliminary single-run results that will be extended before the final version, results are reported form their repositories in \cref{tab:cifar10-community-baselines}.

\begin{table}[t]
    \centering
    \small
    \begin{tabular}{lccc}
        \hline
        Method & Val. acc. (\%) & TTA acc. (\%) & Runs \\
        \hline
        Conv-NS (ours) & $93.17 \pm 0.12$ & $93.98 \pm 0.09$ & 25 \\
        Muon (reshape) & $93.10 \pm 0.14$ & $93.99 \pm 0.14$ & 25 \\
        Matrixizer & $92.96$ & $93.73$ & 1 \\
        Freq-Muon & $92.55$ & $93.06$ & 1 \\
        SGD & $90.47$ & $91.11$ & 1 \\
        \hline
    \end{tabular}
    \caption{\textbf{CIFAR-10 AirBench94 community baseline comparison.} Mean $\pm$ standard deviation over 25 runs for Conv-NS and Muon~(reshape); single-run baselines are preliminary and shown as context.}
    \label{tab:cifar10-community-baselines}
\end{table}

Conv-NS and reshape-based Muon achieve near-identical CIFAR-10 AirBench94 performance and both outperform the single-run SGD, Freq-Muon, and Matrixizer baselines.
Because Matrixizer and Freq-Muon are preliminary community implementations without peer-reviewed standardized evaluations, we use these results as diagnostic context rather than as a central comparison in the main paper.

\section{LMO dense matrix proof}
\label{app:proofLMO}
The objective of Muon is to find for a given layer of weight $W$ the optimal weight update that maximise the loss decrease under a constraint of the maximal variation of the output of the layer. Mathematically, the formulation of the loss decrease is based on the first-order Taylor approximation.
\begin{equation}
\mathcal{L}(W+\Delta_W)
=
\mathcal{L}(W)
+
\left\langle G,\Delta_W\right\rangle_F
+
o\!\left(\|\Delta_W\|\right).
\label{eq:taylordense}
\end{equation}
While the constraint on the maximum variation of the layer output is given by the operator norm:
\begin{equation}
\sup_{\|x\|_2=1}
\left\|Wx-(W+\Delta_W)x\right\|_2 = \left\|\Delta_W\right\|_{\mathrm{op}}
\end{equation}
This lead to a LMO formulation in the space of the dense matrices:
\begin{equation*}
\Delta_W^\star = \underset{\|\Delta_W\|_{\mathrm{op}} \le \eta}{\arg\min} \langle G, \Delta_W \rangle_F
\end{equation*}
The solution of this LMO is proportional to the the polar factor of the gradient defined by 

\begin{equation*}
\operatorname{PF}(G) = \underset{Q}{\operatorname{argmin}} \quad \mathcal{L}_{\mathrm{frob}}(G,Q) \quad \text{subject to} \quad \| Q Q^\top - I \|_F^2 = 0
\end{equation*}
or with $m\leq n$, using the definition of the Stiefel Manifold  $St(m,n)=\left\{Q\in \mathbb{R}^{m\times n}: QQ^T=Id\right\}$, noted $\stief$.

\begin{equation*}
\operatorname{PF}(G) = \underset{Q\in \stief}{\operatorname{argmin}} \quad \mathcal{L}_{\mathrm{frob}}(G,Q) \quad
\end{equation*}

\begin{proof}
    First, $\forall Q\in \stief$, the Frobenius norm is constant and equal to m: $\|Q\|_F^2=tr(QQ^T)=m$.

    Therefore,
    $$
    \|Q-G\|_F^2 = \|Q\|_F^2+\|G\|_F^2 - 2\langle G, Q \rangle_F = m+\|G\|_F^2 - 2\langle G, Q \rangle_F 
    $$
    Thus $PF(G) \in \underset{Q\in \stief}{\operatorname{argmax}} \quad \langle G, Q \rangle_F \quad$, or equivalently $-\eta PF(G) \in \underset{Q\in \eta \stief}{\operatorname{argmin}} \quad \langle G, Q \rangle_F \quad$.

    But in the previous equation, the minimization is over $\eta \stief$, and not the set $\left\{\Delta\in\mathbb{R}^{m\times n}: \|\Delta\|_{\mathrm{op}}\leq \eta\right\}$.

    Since, Linear minimization does not distinguish a set from its convex hull, noted $\operatorname{CH}$, so:
    $$
    -\eta PF(G) \in \underset{Q\in \eta \operatorname{CH}(St(m,n))}{\operatorname{argmin}} \quad \langle G, Q \rangle_F \quad
    $$.

    An important property of the Stiefel Manifold for dense matrices is that: 
    $$
    \operatorname{CH}( St(m,n))=\left\{\Delta\in\mathbb{R}^{m\times n}: \|\Delta\|_{\mathrm{op}}\leq 1\right\}
    $$
    The proof is given in~\citet{journee2010generalized}

Thus
$$
    -\eta PF(G) \in \underset{\|Q\|_{\mathrm{op}}\leq \eta}{\operatorname{argmin}} \quad \langle G, Q \rangle_F \quad
$$.
\end{proof}

\section{convolutional LMO proof}
\label{app:proofconvLMO}
This section presents the convolutional LMO construction. Given a convolutional layer with a kernel $K$, the first-order Taylor approximation is similar to the dense formulatin (\Cref{eq:taylordense}):

\begin{equation}
\mathcal{L}(K+\Delta_K)
=
\mathcal{L}(K)
+
\left\langle G_K,\Delta_K\right\rangle_F
+
o\!\left(\|\Delta_K\|\right).
\label{eq:taylorconv}
\end{equation}

But the constraint on the minimum variation has to be evaluated in the space of the structured linear operator $\mathcal{T}(K)$:

\begin{equation}
\sup_{\|x\|_2=1}
\left\|\mathcal{T}(K)x-\mathcal{T}(K+\Delta_K)x\right\|_2 = \left\|\mathcal{T}(\Delta_K)\right\|_{\mathrm{op}}
\end{equation}
The convolutional LMO formulation is thus:

\begin{equation*}
\Delta_K^\star = \underset{
\substack{
  \Delta_K\in \mathcal{M}_{\mathrm{kernel}}\\
  \|\mathcal{T}(\Delta_K)\|_{\mathrm{op}} \le \eta}}{\arg\min} \langle G_K, \Delta_K \rangle_F
\end{equation*}
Furthermore, under circular padding and fixed spatial resolution, the Frobenius distance between two convolutional operators is proportional to the Frobenius distance between their underlying kernels:
\begin{equation}
\left\langle \mathcal{T}(G_K), \mathcal{T}(\Delta_K) \right\rangle_F
\propto \left\langle G_K, \Delta_K \right\rangle_F.
\end{equation}
This lead to the convolutional LMO formulation provided in \Cref{eq:conv-lmo}.

The link with the convolutional polar factor  $\convPF(G_K)$ (\Cref{eq:conv_op_lmo}) follows the same proof as for dense matrices:

\begin{proof}
    First, $\forall \Delta_K\in \mathcal{M}_{\mathrm{kernel}}$ such that $\mathcal{T}(\Delta_K)\in \stief$, the Frobenius norm of $\mathcal{T}(\Delta_K)$ is constant.

    Therefore,
    $$
    \|\mathcal{T}(\Delta_K)-\mathcal{T}(G_K)\|_F^2 = \|\mathcal{T}(\Delta_K)\|_F^2+\|\mathcal{T}(G_K)\|_F^2 - 2\langle \mathcal{T}(G_K), \mathcal{T}(\Delta_K) \rangle_F = cste - 2\langle \mathcal{T}(G_K), \mathcal{T}(\Delta_K) \rangle_F 
    $$
    Thus $\convPF(G_K) \in \underset{
  \substack{
  \Delta_K\in \mathcal{M}_{\mathrm{kernel}}\\
  \mathcal{T}(\Delta_K)\in \stief}
  }{\operatorname{argmax}} \quad \langle \mathcal{T}(G_K), \mathcal{T}(\Delta_K) \rangle_F \quad$, 
  
  or equivalently $-\eta \convPF(G) \in \underset{
  \substack{
  \Delta_K\in \mathcal{M}_{\mathrm{kernel}}\\
  \mathcal{T}(\Delta_K)\in \eta\stief}
  }{\operatorname{argmin}} \quad \langle \mathcal{T}(G_K), \mathcal{T}(\Delta_K) \rangle_F \quad$.
\end{proof}

Unfortunately the Convex Hull argument is no more an equality in the convolution setup since:
\[
\operatorname{CH}(\mathcal{T}(\mathcal{M}_{\mathrm{kernel}}) \cap \stief) \subseteq \mathcal{T}(\mathcal{M}_{\mathrm{kernel}}) \cap \operatorname{CH}(\stief)=\left\{ \Delta \in \mathcal{M}_{\mathrm{kernel}} \text{ such that } \|\mathcal{T}(\Delta)\|_{op}\leq 1\right\}
\]

The convolutional polar factor  $\convPF$ is thus the optimum on a subspace of the unit ball, inducing an overconstraint.

\end{document}